\documentclass[11pt]{article}
\usepackage{amsmath,amsfonts,amssymb}
\usepackage{amsthm}
\usepackage{xspace}
\usepackage{bm}
\usepackage{hyperref}

\newtheorem{theorem}{Theorem}[section]

\newtheorem{lemma}[theorem]{Lemma}
\newtheorem{claim}[theorem]{Claim}
\newtheorem{proposition}[theorem]{Proposition}

{\theoremstyle{remark} \newtheorem{remark}[theorem]{Remark}}
{\theoremstyle{definition} \newtheorem{definition}[theorem]{Definition}}
\newtheorem{algorithm}{Algorithm}

\newenvironment{proofof}[1]{\begin{proof}[Proof of #1]}{\end{proof}}

\newenvironment{labelalg}[1]{\medskip\noindent{\bf Algorithm #1}}{\medskip}

\newenvironment{labellist}[1][A]
{\begin{list}{{#1}\arabic{enumi}.}{\usecounter{enumi}\addtolength{\leftmargin}{-1.5ex}}}
{\end{list}}

\newcommand{\comment}[1]{}
\newcommand{\remove}[1]{}
\newcommand{\suppress}[1]{}

\DeclareMathOperator{\sign}{sign}
\DeclareMathOperator{\conv}{conv}
\DeclareMathOperator{\aff}{aff}
\DeclareMathOperator{\E}{E}
\DeclareMathOperator{\Var}{Var}
\DeclareMathOperator{\poly}{poly}

\DeclareMathOperator{\Span}{Span}
\DeclareMathOperator{\Tran}{Tran}
\DeclareMathOperator{\rank}{rank}

\DeclareMathOperator{\col}{\Span}
\DeclareMathOperator{\diag}{diag}

\newcommand{\NN}{{\mathbb{N}}}
\newcommand{\RR}{{\mathbb{R}}}

\newcommand{\eps}{\epsilon}

\newcommand{\be}{\begin{equation}}
\newcommand{\ee}{\end{equation}}
\newcommand{\bea}{\begin{eqnarray}}
\newcommand{\eea}{\end{eqnarray}}
\newcommand{\bean}{\begin{eqnarray*}}
\newcommand{\eean}{\end{eqnarray*}}

\newcommand{\veps}{\varepsilon}

\newcommand{\obt}{\overline{\vth}}
\newcommand{\oba}{\overline{\al}}

\newcommand{\Pas}{\mathrm{Pas}}
\newcommand{\op}{\mathrm{op}}

\newcommand{\ez}{\varepsilon}

\newcommand{\R}{\ensuremath{\mathbb R}}
\newcommand{\T}{\ensuremath{\mathcal{T}}}
\newcommand{\sm}{\ensuremath{\setminus}}
\newcommand{\es}{\ensuremath{\emptyset}}
\newcommand{\sse}{\subseteq}

\newcommand{\polylog}{\operatorname{polylog}}

\newcommand{\e}{\ensuremath{\epsilon}}
\newcommand{\gm}{\ensuremath{\gamma}}
\newcommand{\ld}{\ensuremath{\lambda}}
\newcommand{\Ld}{\ensuremath{\Lambda}}
\newcommand{\tld}{\ensuremath{{\tilde\lambda}}}
\newcommand{\tLd}{\ensuremath{\tilde\Lambda}}
\newcommand{\kp}{\ensuremath{\kappa}}
\newcommand{\al}{\ensuremath{\alpha}}
\newcommand{\tal}{\ensuremath{\tilde\alpha}}
\newcommand{\tht}{\ensuremath{\theta}}
\newcommand{\dt}{\ensuremath{\delta}}
\newcommand{\Dt}{\ensuremath{\Delta}}
\newcommand{\sg}{\ensuremath{\sigma}}
\newcommand{\w}{\ensuremath{\omega}}

\newcommand{\vro}{\ensuremath{\varrho}}
\newcommand{\tg}{\ensuremath{\tilde g}}
\newcommand{\tA}{\ensuremath{\tilde A}}
\newcommand{\tM}{\ensuremath{\tilde M}}
\newcommand{\tR}{\ensuremath{\tilde R}}
\newcommand{\tV}{\ensuremath{\tilde V}}
\newcommand{\tG}{\ensuremath{\tilde G}}

\newcommand{\tr}{\ensuremath{\tilde r}}
\newcommand{\tp}{\ensuremath{\tilde p}}
\newcommand{\tP}{\ensuremath{\tilde P}}
\newcommand{\tq}{\ensuremath{\tilde q}}

\newcommand{\tw}{\ensuremath{\tilde w}}
\newcommand{\tnu}{\ensuremath{\tilde\nu}}

\newcommand{\hp}{\ensuremath{\hat p}}
\newcommand{\hP}{\ensuremath{\hat P}}
\newcommand{\hA}{\ensuremath{\hat A}}
\newcommand{\hr}{\ensuremath{\hat r}}
\newcommand{\hw}{\ensuremath{\hat w}}
\newcommand{\hc}{\ensuremath{\hat c}}
\newcommand{\hal}{\ensuremath{\hat\alpha}}
\newcommand{\bal}{\ensuremath{\bar\alpha}}
\newcommand{\bw}{\ensuremath{\bar w}}
\newcommand{\assign}{\ensuremath{\leftarrow}}

\newcommand{\learn}{\ensuremath{\mathsf{Learn}}\xspace}
\newcommand{\procacc}{\ensuremath{\varsigma}}
\newcommand{\procerrp}{\ensuremath{\varepsilon}}
\newcommand{\onedzeta}{\ensuremath{\tau}}
\newcommand{\onederrp}{\ensuremath{\psi}}
\newcommand{\cn}{\ensuremath{f}}
\newcommand{\iso}{} 

\newcommand{\bdy}{\boldsymbol{y}}
\newcommand{\bdz}{\boldsymbol{z}}
\newcommand{\dist}[1]{\ensuremath{D^{#1}}}
\newcommand{\dtv}{\ensuremath{d_{\text{TV}}}}
\newcommand{\bdth}{\boldsymbol{\theta}}

\newcommand{\vth}{\vartheta}

\newcommand{\tvth}{\tilde{\vartheta}}

\newcommand{\real}{\ensuremath{\mathrm{Re}}}
\newcommand{\da}{\dagger}

\title{Learning Mixtures of Arbitrary Distributions over Large Discrete Domains}

\author{
     Yuval Rabani\thanks{The Rachel and Selim Benin School
     of Computer Science and Engineering and the Center of
     Excellence on Algorithms, The Hebrew University of
     Jerusalem, Jerusalem 91904, Israel. Email: {\tt
     yrabani@cs.huji.ac.il}.}
\and
     Leonard J.\ Schulman\thanks{
     Caltech, Pasadena, CA 91125, USA. Supported in part by NSF
     CCF-1038578, NSF CCF-0515342, NSA H98230-06-1-0074, and
     NSF ITR CCR-0326554. Email: {\tt schulman@caltech.edu}.}
\and
     Chaitanya Swamy\thanks{
     Dept. of Combinatorics and Optimization, Univ. Waterloo,
     Waterloo, ON N2L 3G1, Canada. Supported in part by NSERC grant
     32760-06, an NSERC Discovery Accelerator Supplement Award, and an Ontario Early
     Researcher Award.
     Email: {\tt cswamy@math.uwaterloo.ca}.}
}

\date{}

\begin{document}
\sloppy

\maketitle

\begin{abstract}
We give an algorithm for learning a mixture of {\em unstructured}
distributions. This problem arises in various unsupervised learning
scenarios, for example in learning {\em topic models} from a
corpus of documents spanning several topics. We show how to
learn the constituents of a mixture of $k$ arbitrary distributions over  
a large discrete domain $[n]=\{1,2,\dots,n\}$ and the mixture weights, 
using $O(n\polylog n)$ samples. (In the topic-model learning setting, the mixture
constituents correspond to the topic distributions.)  

This task is information-theoretically impossible for $k>1$ under 
the usual sampling process from a mixture distribution. However, there 
are situations (such as the above-mentioned topic model case) in which 
each sample point consists of several observations from the same mixture 
constituent. This number of observations, which we call the {\em ``sampling aperture''},  
is a crucial parameter of the problem. 

We obtain the {\em first} bounds for this mixture-learning problem 
{\em without imposing any assumptions on the mixture constituents.} 
We show that efficient learning is possible exactly
at the information-theoretically least-possible aperture of $2k-1$. 
Thus, we achieve near-optimal dependence on $n$ and optimal aperture. 
While the sample-size required by our algorithm depends exponentially on $k$, we prove
that such a dependence is {\em unavoidable} when one considers general mixtures. 

A sequence of tools contribute to the algorithm, such as concentration 
results for random matrices, dimension reduction, moment estimations,
and sensitivity analysis.
\end{abstract}

\section{Introduction} \label{intro}
We give an algorithm for learning a mixture of {\em unstructured}
distributions. More specifically, we consider the problem of learning 
a mixture of $k$ arbitrary distributions over a large finite domain 
$[n]=\{1,2,\dots,n\}$. 
This finds applications in various unsupervised 
learning scenarios including {\em collaborative filtering}~\cite{HP99}, and learning
{\em topic models} from a corpus of documents spanning several
topics~\cite{PRTV97,Blei12}, which 
is often used as the prototypical motivating example for this problem. Our 
goal is to learn the probabilistic model that is hypothesized to 
generate the observed data. In particular, we learn the constituents 
of the mixture and their weights in the mixture.
(In the topic models application, the mixture constituents
are the topic distributions.)

It is information-theoretically impossible to reconstruct the mixture 
model from single-snapshot samples. 
Thus, our work relies on multi-snapshot samples. To illustrate, in the (pure
documents) topic model introduced in~\cite{PRTV97}, each document consists of a 
{\em bag of words} generated by selecting a topic with probability proportional to
its mixture weight and then taking independent samples from this topic's distribution
(over words); so $n$ is the size of the vocabulary and $k$ is the number of topics. 
Notice that typically $n$ will be quite large, and substantially larger than $k$.
Also, clearly, if very long documents are available, the problem becomes easy, as each
document already provides a very good sample for the distribution of its topic. 
Thus, it is desirable to keep the dependence of the sample size on $n$ as low as possible,
while at the same time minimize what we call the {\em aperture}, which is the number of
snapshots per sample point (i.e., words per document). 
{These parameters govern both the applicability of an algorithm and
its computational complexity.}

\vspace{-1ex}
\paragraph{Our results.}
We provide the {\em first} bounds for the mixture-learning problem  
{\em without making any limiting assumptions} on the mixture constituents. 
Let probability distributions $p^1,\ldots,p^k\in\Dt^{n-1}$ denote the $k$-mixture
constituents, where $\Dt^{n-1}$ is the $(n-1)$-simplex, and $w_1,\ldots,w_k$ denote the
mixture weights. 
Our algorithm uses 
\begin{equation}
O\left(\frac{k^3n\ln n}{\e^6}\right) +
O\left(\frac{k^2n\ln^6 n\ln\bigl(\tfrac{k}{\e}\bigr)}{\e^4}\right) +
O\left(\frac{k}{\e}\right)^{O(k^2)}
\label{sampsize}
\end{equation}
documents (i.e., samples) 
and reconstructs with high probability (see Theorem~\ref{mainthm}) each mixture
constituent up to $\ell_1$-error $\e$, and each mixture weight up to additive error
$\e$. We make no assumptions on the constituents.  
The asymptotic notation hides factors that are polynomial in $w_{\min}:=\min_t w_t$ and  
the {\em ``width"} of the mixture 
(which intuitively measures the minimum variation distance between any two constituents).     
The three terms in \eqref{sampsize} correspond to the requirements for the number of  
$1$-, $2$-, and $(2k-1)$-snapshots respectively. So we need aperture 
$2k-1$ only for a small part of the sample (and this is necessary). 

Notably, we achieve {\em near-optimal dependence on $n$ and optimal aperture}.   
To see this, and put our bounds in perspective, notice importantly that we recover the 
mixture constituents within {\em $\ell_1$-distance} $\e$.
One needs $\Omega\bigl(n/\e^2\bigr)$ samples to learn even a single arbitrary distribution
over $[n]$ (i.e., $k=1$) within $\ell_1$-error $\e$; for larger $k$ but fixed aperture
(independent of $n$), a sample size of $\Omega(n)$ is {\em necessary} to recover even the
expectation of the mixture distribution with constant $\ell_1$-error.  
On the other hand, aperture $\Omega\bigl((n+k^2) \log nk\bigr)$ is sufficient for   
algorithmically trivial recovery of the model with constant $\ell_\infty$ error
using few samples. 
Restricting the aperture to $2k-2$ makes recovery {\em impossible} to arbitrary accuracy
(without additional assumptions): we show that there are two far-apart $k$-mixtures
that generate exactly the same aperture-$(2k-2)$ sample distribution; moreover, we prove
that with $O(k)$ aperture, an exponential in $k$ sample size is {\em necessary} for
arbitrary-accuracy reconstruction.
These lower bounds hold even for $n=2$, and hence apply to arbitrary mixtures even if we
allow $O(k\log n)$ aperture. Also, they apply even if we only
want to construct a $k$-mixture source that is close in transportation distance to the
true $k$-mixture source (as opposed to recovering the parameters of the true mixture).
Section~\ref{lbound} presents these lower bounds. 
(Interestingly, an exponential in $k$ sample-size lower
bound is also known for the problem of learning a mixture of $k$ Gaussians~\cite{MV10},
but this lower bound applies for the parameter-recovery problem and not for
{reconstructing a mixture that is close to the true Gaussian mixture.)}

Our work yields new insights into the mixture-learning problem 
that nicely complements the recent interesting work of~\cite{AGM12,AHK12,AFHKL12}. 
These papers posit certain assumptions on the mixture constituents, 
use constant aperture, and obtain incomparable sample-size bounds: 
they recover the constituents up to $\ell_2$ or $\ell_\infty$ error 
using sample size that is $\poly(k)$ and sublinear in (or independent of) $n$. 
An important new insight revealed by our work is that such bounds of constant aperture and 
$\poly(k)$ sample size are {\em impossible} to achieve for arbitrary mixtures. 
Moreover, if we seek to achieve $\ell_1$-error $\e$, there are inputs 
for which their sample size is $\Omega(n^3)$ (or worse, again ignoring dependence on
$w_{\min}$ and ``width''; see Appendix~\ref{comparison}).     
This is a significantly poorer dependence on $n$ 
compared to our near-linear dependence (so our bounds are better when $n$ is
large but $k$ is small). 
To appreciate a key distinction between our work and~\cite{AGM12,AHK12,AFHKL12},
observe that with $\Omega(n^3)$ samples, the entire distribution on 3-snapshots can
be estimated fairly accurately; 
the challenge in~\cite{AGM12,AHK12,AFHKL12} is therefore to recover the model from this
relatively noiseless data.    
In contrast, a major challenge for achieving $\ell_1$-reconstruction with
$O(n\polylog n)$ samples 
is to ensure that the error remains bounded despite the presence 
{of very noisy data due to 
the small sample size, and we develop suitable machinery to achieve this.}

\medskip
We now give a rough sketch of our algorithm (see Section~\ref{algo}) and the ideas behind
its analysis (Section~\ref{analysis}).   
Let $P=(p^1,\ldots,p^k)$, $r=\sum_t w_tp^t$ be the expectation of the mixture, 
and $k'=\rank(p^1-r,\ldots,p^k-r)$. 
We first argue that it suffices to focus on isotropic mixtures (Lemma~\ref{lm: isotropy}).
Our algorithm reduces the problem to the problem of learning 
{\em one-dimensional mixtures}. Note that this is a special case of the general learning
problem that we need to be able to solve (since we do not make any assumptions about the
rank of $P$).
We choose $k'$ random lines that are
close to the affine hull, $\aff(P)$, of $P$ and ``project'' the mixture on to these $k'$
lines. We learn each projected mixture, which is a one-dimensional mixture-learning
problem, and combine the inferred projections on these $k'$ lines to obtain $k$
points that are close to $\aff(P)$. Finally, we project these $k'$ points on to
$\Dt^{n-1}$ to obtain $k$ distributions over $[n]$, which we argue are close (in
$\ell_1$-distance) to $p^1,\ldots,p^k$.

Various difficulties arise in implementing this plan. 
We first learn a good approximation to $\aff(P)$ using spectral techniques and
2-snapshots. We use ideas similar to~\cite{McS01,AFKMS01,KS08}, but our challenge is to 
show that the covariance matrix $A=\sum_t w_t(p^t-r)(p^t-r)^{\dagger}$ 
can be well-approximated by the empirical covariance matrix 
with only $O(n\ln^6 n)$ 2-snapshots.
A random orthonormal basis of the learned affine space supplies the $k'$ lines 
on which we project our mixture.
Of course, we do not know $P$, so ``projecting'' on to a basis
vector $b$ actually means that we project snapshots from $P$
on to $b$ by mapping item $i$ to $b_i$. 
For this to be meaningful, we need to ensure that if the mixture constituents are far
apart in variation distance then their projections $(b^\dagger p^t)_{t\in[k]}$ are also
well separated 
relative to the spread of the support $\{b_{1},\ldots b_{n}\}$ of the one-dimensional
distribution. In general, we can only claim a relative separation of
$\Theta\bigl(\frac{1}{\sqrt{n}}\bigr)$ 
(since $\min_{t\neq t'}\|p^t-p^{t'}\|_2$ may be $\Theta\bigl(\frac{1}{\sqrt{n}}\bigr)$).  
We avoid this via a careful balancing act: we prove (Lemma~\ref{alldir}) 
that the $\ell_\infty$ norm of unit vectors in $\aff(P)$ is
$O\bigl(\frac{1}{\sqrt{n}}\bigr)$, and argue that this isotropy property suffices
since $b$ is close to $\aff(P)$.   

Finally, a key ingredient of our algorithm (see Section~\ref{onedim}) is to show how to
solve the one-dimensional mixture-learning problem and
learn the real projections $(b^\dagger p^t)_{t\in[k]}$ from 
the projected snapshots. This is technically the most difficult step and the
one that requires aperture $2k-1$ (the smallest aperture at which this is
possible).   
We show that the projected snapshots on $b$ yield empirical moments of a related 
distribution 
and use this to learn the projections and the mixture weights via a method of moments
(see, e.g.,~\cite{FeldmanOS05,FOS06,KMV10,BS10,MV10,AHK12}).    
One technical difficulty is that variation
distance in $\Delta^{n-1}$ translates to transportation distance~\cite{Villani03}
in the one-dimensional projection. We use a combination of convex programming and
numerical-analysis techniques 
to learn the projections from the empirical ``directional'' moments.
In the process, we establish some novel properties about the {\em moment curve}---an 
object that plays a central role in convex and polyhedral
geometry~\cite{Barvinok02}---that may be of independent interest.  

\vspace{-1ex}
\paragraph{Related work.}
The past decade has witnessed tremendous progress in the theory
of learning statistical mixture models. The most striking example is
that of learning mixtures of high dimensional Gaussians. Starting with
Dasgupta's groundbreaking paper~\cite{Das99}, a long sequence of
improvements~\cite{DS00,AK01,VW02,KSV05,AM05,FOS06,BV08} 
culminated in the recent results~\cite{KMV10,BS10,MV10} that essentially 
resolve the problem in its general form. In this vein, other highly structured 
mixture models, such as mixtures of discrete product
distributions~\cite{KMRRSS94,FreundM99,CryanGG02,FeldmanOS05,CHRZ07,CR08a}
and similar
models~\cite{CryanGG02,BGK04,MR05,KSV05,DHKS05,CR08b,DaskalakisDS12},
have been studied intensively. One important difference between this line of
work and ours is that the structure of those mixtures enables learning using
single-snapshot samples, whereas this is impossible in our case. Another 
interesting difference between our setting and the work on structured models
(and this is typical of most results on PAC-style learning) 
is that the amount of information in each sample point is roughly in the same 
ballpark as the information needed to describe the model. In our setting, the
amount of information in each sample point is exponentially sparser than the 
information needed to describe the model to good accuracy. Thus, the topic-model learning 
problem motivates the natural question of inference from sparse samples. This issue is
also encountered in collaborative filtering; see~\cite{KS08} for some related theoretical 
problems. 

Recently and independently,~\cite{AGM12,AHK12,AFHKL12} have considered much the same 
question as ours.\footnote{An earlier stage of this work, including the case
$k=2$ as well as some other results that are not subsumed by this paper, dates
to 2007. The last version of that phase has been posted since May 2008
at~\cite{RSS08}. The extension to arbitrary $k$ is from last year.} 
They make certain assumptions about the mixture constituents which makes it   
possible to learn the mixture with constant aperture and $\poly(n,k)$ sample
size (for $\ell_1$-error). In comparison with our work, their sample bounds are 
attractive in terms of $k$ but come at the expense of added assumptions (which are
necessary), and have a worse dependence on $n$. 

The assumptions in~\cite{AGM12,AHK12,AFHKL12} impose some limitations on the applicability
of their algorithms. To understand this, 
it is illuminating to consider the case where all the $p^t$s lie on a line-segment in 
$\Dt^{n-1}$ as an illustration. This poses no problems for our algorithm: we recover
the $p^t$s along with their mixture weights. However, as we show below, the algorithms
in~\cite{AGM12,AHK12,AFHKL12} all fail to reconstruct this mixture.
Anandkumar et al.~\cite{AHK12} solve the same problem that we consider, under the
assumption that $P$ (viewed as an $n\times k$ matrix) has rank $k$.  
This is clearly violated here, rendering their algorithm inapplicable.  
The other two papers~\cite{AGM12,AFHKL12} consider the setting where each multi-snapshot
is generated from a combination of mixture constituents 
\cite{PRTV97,Hof99}: first a convex combination $\ld\in\Dt^{k-1}$ is sampled from a
mixture distribution $\T$ on $\Dt^{k-1}$, then the snapshot is generated by sampling from 
the distribution $\sum_{t=1}^k \ld_tp^t$.   
The goal is to learn the mixture constituents and the mixture distribution.  
(The problem we consider is the special case where $\T$ places weight $w^t$ on the $t$-th  
vertex of $\Dt^{k-1}$.)
\cite{AGM12} posits a {\em $\rho$-separability} assumption on the mixture constituents,
wherein each $p^t$ has a unique ``anchor word'' $i$ such that $p^t_i\geq\rho$ and
$p^{t'}_i=0$ for every $t'\neq t$,
whereas~\cite{AFHKL12} weakens this to the requirement that $P$ has rank $k$.
Both papers handle the case where $\T$ is the Dirichlet distribution (which
gives the latent Dirichlet model~\cite{BNJ03}); \cite{AGM12} obtains results for
other mixture distributions as well.  

In order to apply these algorithms, we can view the input as being specified by two
constituents, $x$ and $y$, which are the end points of the line segment; $\T$ then places 
weight $w_t$ on the convex combination $(\ld_t,1-\ld_t)^\dagger$, where 
$p^t=\ld_tx+(1-\ld_t)y$. 
This $\T$ is far from the Dirichlet distribution, so~\cite{AHK12} does not apply
here. Suppose that $x$ and $y$ satisfy the $\rho$-separability condition. 
(Note that $\rho$ may only be $O\bigl(\frac{1}{n}\bigr)$, even if 
$x$ and $y$ have {\em disjoint} supports.)
We can then apply the algorithm of Arora et al.~\cite{AGM12}. But this {\em does not}
recover $\T$; it returns the ``topic correlation'' matrix 
$\E_{\T}[\ld\ld^\dagger]$, which does not reconstruct the mixture $(w,P)$. 

This limitation should not be surprising since~\cite{AGM12} uses constant
aperture. Indeed, \cite{AGM12} notes that it is impossible to reconstruct $\T$ with
arbitrary accuracy (with any constant aperture) even if one knows the constituents $x$ and $y$.   
In this context, we remark that our earlier work~\cite{RSS08}
uses the approach presented in this paper and solves the problem for 
{\em arbitrary} mixtures of two distributions, 
yielding a crisp statement about the tradeoff between the sampling aperture
and the accuracy with which $\T$ can be learnt. 

Our methods bear some resemblance with the recent independent work of Gravin et
al.~\cite{GravinLPR12} who consider the problem of recovering the vertices of a polytope
from its directional moments. \cite{GravinLPR12} solves this problem for a polynomial
density function assuming that exact directional moments are available; they do not
perform any sensitivity analysis for measuring the error in their output if one has noisy
information. In contrast, we solve this problem given only noisy empirical moment 
statistics and using much smaller aperture, albeit when the polytope is a subset of the
$(n-1)$-simplex and the distribution is concentrated on its vertices. 

Finally, it is also pertinent to compare our mixture-learning problem with 
the problem of learning a mixture of product distributions 
(e.g.,~\cite{FeldmanOS05}). Multi-snapshot samples can be thought
of as single-snapshot samples from the power distribution on $[n]^K$,
where $K$ is the aperture. The product distribution literature
typically deals with samples spaces that are the product of many
small cardinality components, whereas our problem
deals with samples spaces that are the product of few large
cardinality components.

\section{Preliminaries} \label{prelim}

\subsection{Mixture sources, snapshots, and projections} \label{mix-defn}
Let $[n]$ denote $\{1,2,\dots,n\}$, and $\Dt^{n-1}$ denote the $(n-1)$-simplex
$\{x\in\R_{\geq 0}^n: \sum_i x_i=1\}$.
A {\em $k$-mixture source $(w,P)$ on $[n]$}
consists of $k$ mixture constituents
$P=(p^1,\dots,p^k)$, where $p^t$ has support $[n]$ for all $t\in[k]$,
along with the corresponding mixture weights
$w=(w_1,\ldots,w_k)\in\Delta^{k-1}$.
An $m$-snapshot from $(w,P)$ is obtained by
choosing $t\in [k]$ according to the distribution $w$,
and then choosing $i\in [n]$ $m$ times independently
according to the distribution $p^t$. The probability
distribution on $m$-snapshots is thus a mixture of $k$
power distributions on the product space $[n]^m$.
We also consider mixture sources whose constituents are distributions on $\RR$.
A {\em $k$-mixture source $(w,P)$ on $\RR$} consists of $k$ mixture constituents
$P=(p^1,p^2,\dots,p^k)$, where each $p^t$ is a probability distribution on $\RR$, along
with corresponding mixture weights $w=(w_1,\ldots,w_k)\in\Delta^{k-1}$.

Given a distribution $p$ on $[n]$ and a vector $x\in\R^n$, we define the projection of $p$
on $x$, denoted $\pi_x(p)$, to be the discrete distribution on $\R$ that assigns
probability mass $\sum_{i:x_i=\beta}p_i$ to $\beta\in\R$. (Thus, $\pi_x(p)$ has support
$\{x_1,\ldots,x_n\}$ and $\E[\pi_x(p)]=x^\dagger p$.)
Given a $k$-mixture source $(w,P)$ on $[n]$, we define the projected $k$-mixture source
$(w,\pi_x(P))$ on $\RR$ to be the $k$-mixture source on $\R$ given by
$\bigl(w,(\pi_x(p^1),\ldots,\pi_x(p^k))\bigr)$.

We also denote by $(w,\E[\pi_x(P)])$ the distribution that assigns probability
mass $w_t$ to $\E[\pi_x(p^t)]=x^\dagger p^t$ for all $t\in[k]$.
This is an example of what we call a {\em $k$-spike distribution}, which is a distribution
on $\R$ that assigns positive probability mass to $k$ points in $\R$.

\subsection{Transportation distance for mixtures} \label{trans-defn} 
Let $\bigl(w,(p^1,\ldots,p^k)\bigr)$ and
$\bigl(\tw,(\tp^1,\ldots,\tp^\ell)\bigr)$ be $k$- and $\ell$- mixture sources on $[n]$
respectively.  
The {\em transportation distance} 
(with respect to the total variation distance $\frac 1 2 \|x - y\|_1$ on measures on
$\Dt^{n-1}$) between these two sources, denoted by $\Tran(w,P;\tw,\tP)$, is the
optimum value of the following linear program (LP).
$$
\min \quad \sum_{i=1}^k\sum_{j=1}^\ell x_{ij}\cdot\frac{1}{2}\|p^i-\tp^j\|_1
\quad\ \text{s.t.} \quad\
\sum_{j=1}^\ell x_{ij}=w_i \quad \forall i\in[k], \quad  
\sum_{i=1}^k x_{ij}=\tw_j \quad \forall j\in[\ell], \quad
x\geq 0.
$$
The transportation distance $\Tran(w,\al;\tw,\tal)$ between a $k$-spike 
distribution $\bigl(w,\al=(\al_1,\ldots,\al_k)\bigr)$ and an $\ell$-spike distribution  
$\bigl(\tw,\tal=(\tal_1,\ldots,\tal_\ell)\bigr)$ is defined as the optimum value of the
above LP with the objective function replaced by
$\sum_{i\in[k],j\in[\ell]}x_{ij}|\al_i-\tal_j|$. 

\subsection{Perturbation results and operator norm of random matrices} \label{pert-res}

\begin{definition}
The \textit{operator norm} of $A$ (induced by the $\ell_2$ norm) is defined by
$\|A\|_\op=\max_{x\neq 0}\frac{\|Ax\|_2}{\|x\|_2}.$
\noindent The {\it Frobenius norm} of $A=(A_{i,j})$ is defined by
$\|A\|_F=\sqrt{\sum_{i,j}A_{i,j}^2}$.
\end{definition}

\begin{lemma}[\textnormal{Weyl; see Theorem 4.3.1 in~\cite{HornJ85}}] \label{evper}
Let $A$ and $B$ be $n\times n$ matrices such that $\|A-B\|_\op\leq\rho$.
Let $\ld_1(A)\geq\ldots\geq\ld_n(A)$, and $\ld_1(B)\geq\ldots\geq\ld_n(B)$ be the sorted
list of eigenvalues of $A$ and $B$ respectively. Then $|\ld_i(A)-\ld_i(B)|\leq\rho$ for
all $i=1,\ldots,n$.
\end{lemma}

\begin{lemma} \label{lm: rescaling} \label{prjerr} 
Let $A, B$ be $n \times n$ positive semi-definite (PSD) matrices whose nonzero
eigenvalues are at least $\ez>0$. Let $\Pi_A$ and $\Pi_B$ be the projection
operators onto the column spaces of $A$ and $B$ respectively.
Let $\|A-B\|_\op \leq \rho$. Then
$\| \Pi_A - \Pi_B \|_\op \leq \sqrt{4\rho/\ez}$.
\end{lemma}

\begin{proof} 
Note that $A \Pi_A = A$, $\Pi_A^2=\Pi_A$, $B \Pi_B = B$, and $\Pi_B^2=\Pi_B$. Let $x$ be a
unit vector. Since $\|(A-B)\|_\op\leq \rho$ and $\Pi_B$ is a contraction,
$\|(A-B) \Pi_B x \| \leq \rho\|\Pi_Bx\|\leq\rho$. 
Now note that $(A-B) \Pi_B x = A \Pi_B x - B x$ so by the triangle inequality, we have
$\|A \Pi_B x - Ax \| \leq 2\rho$. 
Now we can also write
$A \Pi_B x - Ax = A (\Pi_B-I)x = A (\Pi_A\Pi_B-\Pi_A)x$. 
Since $A$ here is acting on a vector that has already been projected down
by $\Pi_A$, we can conclude
$$
2\rho \geq \|A \Pi_B x - Ax \| =
\| A (\Pi_A\Pi_B-\Pi_A)x \| \geq \ez \|(\Pi_A\Pi_B-\Pi_A)x \|.
$$
Thus, $2 \rho / \ez \geq \|(\Pi_A-\Pi_A\Pi_B)x \|$. 
By the symmetric argument we also can write
$2 \rho / \ez \geq \|(\Pi_B-\Pi_B\Pi_A)x \|$. 
Adding these and applying the triangle inequality we have
$$
4 \rho / \ez \geq \|(\Pi_A-\Pi_A\Pi_B+\Pi_B-\Pi_B\Pi_A)x \| =
 \| (\Pi_A^2 -\Pi_A\Pi_B -\Pi_B\Pi_A + \Pi_B^2)x \| =
 \|(\Pi_A-\Pi_B)^2 x \|. 
\qedhere
$$
\end{proof}

\begin{theorem}[\cite{Vu05}] \label{rmat}
For every $\mu>0$, there is a constant $\kp=\kp(\mu)=O(\mu)>0$ such that the following
holds. Let $X_{i,j}, 1\leq i\leq j\leq n$ be independent random variables with
$|X_{ij}|\leq K$, $\E[X_{i,j}]=0$, and $\Var(X_{i,j})\leq\sg^2$ for all $i,j\in[n]$, where 
$\sg\geq\kp^2n^{-1/2}K\ln^2 n$. Let $A$ be the symmetric matrix with
entries $A_{i,j}=X_{\min(i,j),\max(i,j)}$ for all $i,j\in[n]$. Then,
$\Pr\bigl[\|A\|_\op\leq 2\sg\sqrt{n}+\kp(K\sg)^{1/2}n^{1/4}\ln n\bigr]\geq 1-n^{-\mu}$.
\end{theorem}

\section{Our algorithm} \label{algo}
We now describe our algorithm that uses 1-, 2-, and $(2k-1)$-snapshots from 
the mixture source $(w,P)$.
Given a matrix $Z$, we use $\col(Z)$ to denote the column space of $Z$.
Let $r=\sum_{t=1}^k w_tp^t$ denote the 1-snapshot distribution of $(w,P)$. Let $M$ be the
$n\times n$ symmetric matrix representing the 2-snapshot distribution of $(w,P)$; so
$M_{i,j}$ is the probability of obtaining the 2-snapshot $(i,j)\in[n]^2$. Let $R=rr^\dagger$.

\begin{proposition}
$M = \sum_{t=1}^k w_tp^tp^{t\dagger}=R + A$, where $A=\sum_{t=1}^k w_t(p^t - r) (p^t - r)^\dagger$.
\end{proposition}

Note that $M$ and $A$ are both PSD. 
We say that $(w,P)$ is {\em $\zeta$-wide} if 
(i) $\|p-q\|_2\geq\frac{\zeta}{\sqrt{n}}$ for any two distinct $p,q\in P$; and
(ii) the smallest non-zero eigenvalue of $A$ is at least
$\zeta^2\|r\|_\infty\geq\frac{\zeta^2}{n}$. 
We assume that $w_{\min}:=\min_t w_t>0$. 
Let $k'=\rank(A)\leq k-1$. 
{It is easy to estimate $r$ using Chernoff bounds (Lemma~\ref{chernoff}).}

\begin{lemma}\label{pr: r estimation} \label{rest}
For every $\mu\in\NN$ and every $\sg > 0$, 
if we use $N\ge\frac{8(\mu+2)}{\sg^3}\cdot n\ln n$ independent $1$-snapshots and set
$\tr_i$ to be the frequency of $i$ in these 1-snapshots for all $i\in[n]$, then with
probability at least $1 - n^{-\mu}$ the following hold.
\vspace{-1ex}
\begin{equation}
(1-\sg)r_i \leq  \tr_i \leq(1+\sg)r_i \quad 
\text{$\forall i$ with $r_i\geq\frac{\sg}{2n}$}, \qquad
\tr_i \leq  (1+\sg)\sg/2n \quad 
\text{$\forall i$ with $r_i<\frac{\sg}{2n}$.}
\label{eqrest}
\end{equation}
\end{lemma}

It will be convenient in the sequel to assume that our mixture
source $(w,P)$ is {\em isotropic}, by which we mean that
$\frac{1}{2n}\leq r_i\leq\frac{2}{n}$ for all $i\in[n]$; notice that this implies that
$p^t_i\leq\frac{2}{w_{\min}n}$ for all $i\in[n]$.
We show below that this can be assumed at the expense of a small additive error.

\begin{lemma}\label{lm: isotropy}
Suppose that we can learn, with probability $1-\frac{1}{\w}$, the constituents of an 
isotropic $\zeta$-wide $k$-mixture source on $[n]$ to within transportation distance
$\eps$ using $N_1^{\iso}(n;\zeta,\w,\e)$, $N_2^{\iso}(n;\zeta,\w,\e)$, and
$N_{2k-1}^{\iso}(n;\zeta,\w,\e)$ 1-, 2-, and $(2k-1)$-snapshots respectively.
Then, we can learn, with probability $1-O\bigl(\frac{1}{\w}\bigr)$, the constituents of
an arbitrary $\zeta$-wide $k$-mixture source $(w,P)$ on $[n]$ to within transportation
distance $2\eps$ using $O\bigl(\frac{\ln\w}{\sg^3}\cdot n\ln n\bigr)+
6\w N_1^{\iso}\bigl(\frac{n}{\sg},\frac{\zeta}{2},\w,\e\bigr)$, 
$6\w N_2^{\iso}\bigl(\frac{n}{\sg},\frac{\zeta}{2},\w,\e\bigr)$, and 
$6\w N_{2k-1}^{\iso}\bigl(\frac{n}{\sg},\frac{\zeta}{2},\w,\e\bigr)$ 1-, 2-, and
$(2k-1)$-snapshots respectively, 
where $\sg=\frac{\e\zeta^2}{32kw_{\min}}$.
\end{lemma}

\begin{proof}
Given $(w,P)$, we first compute an estimate $\tr$ satisfying \eqref{eqrest},
where $\mu=2+\ln \w$, using $O\bigl(\frac{\ln\w}{\sg^3}\cdot n\ln n\bigr)$ 1-snapshots.    
We assume in the sequel that \eqref{eqrest} holds.
Consider the following modification of the mixture constituents. 
We eliminate items $i$ such that $\tilde{r}_i<\frac{2\sg}{n}$. Each remaining item $i$
is ``split'' into $n_i=\lfloor n\tilde{r}_i/\sg\rfloor$ items, and the probability of $i$
is split equally among its copies. The mixture weights are unchanged. 
From \eqref{eqrest}, we have that $r_i<\frac{4\sg}{n}$ if $i$ is
eliminated. So the total weight of eliminated items is at most $4\sg$.
Let $n'=\sum_{i:\tr_i\geq 2\sg/n}n_i\leq\frac{n}{\sg}$ be the number of new items.
Let $\hP=(\hp^1,\ldots,\hp^k)$ denote the modified mixture constituents, and $\hr$ denote
the distribution of the modified 1-snapshots. 
We prove below that the modified mixture $(w,\hP)$ is 
isotropic and $\zeta/2$-wide. 

We use the algorithm for isotropic mixture sources to learn $(w,\hP)$ within
transportation distance $\e$, using the following procedure to sample $m$-snapshots from
$(w,\hP)$. We obtain an $m$-snapshot from $(w,P)$. We eliminate this snapshot if it
includes an eliminated item; otherwise, each item $i$ in the snapshot is replaced by one
of its $n_i$ copies, 
chosen uniformly at random (and independently of previous such choices).
From the inferred modified mixture source, we can obtain an estimate of the original
mixture source by aggregating, for each inferred mixture constituent, the probabilities of
the items that we split, and setting the probability of each eliminated item to $0$. This 
degrades the quality of the solution by the weight of the eliminated items, which
is at most an additive $4\sg\leq\e$ term in the transportation distance.

The probability that an $m$-snapshot from $(w,P)$ survives is at least 
$(1 -4\sg)^m\geq\frac{1}{2}$ for $m\leq 2k-1$. Therefore, with probability at least
$1-\frac{1}{3\w}$, we need at most $6\w N$ $m$-snapshots from $(w,P)$ to obtain $N$
$m$-snapshots from $(w,\hP)$. (If we violate this bound, we declare failure.) Thus, we   
use at most the stated number of 1-, 2-, and $(2k-1)$-snapshots from $(w,P)$
and succeed with probability $1-O\bigl(\frac{1}{\w}\bigr)$. 

We conclude by showing that $(w,\hP)$ is isotropic and $\zeta/2$-wide.
Let $S=\{i\in[n]: \tr_i<2\sg/n\}$ denote the set of eliminated items.
Recall that $\tilde{r}$ satisfies \eqref{eqrest}. So we have
$\frac{31}{32}\leq\frac{\tr_i}{r_i}\leq\frac{33}{32}$ for every non-eliminated item.  
We use $i_\ell$, where $\ell=1,\ldots,n_i$, to denote a new item obtained by splitting
item $i$. Define $n_i=0$ if $i$ is eliminated.

The number $n'$ of new items is at most $\frac{n}{\sg}$ and at least 
$\sum_{i\notin S}\frac{2}{3}\cdot\frac{n\tr_i}{\sg}
\geq\frac{2}{3}\cdot\frac{n}{\sg}\cdot(1-2\sg)\geq\frac{5n}{8\sg}$.
Let $K=\sum_{i\notin S}r_i\geq 1-4\sg\geq 7/8$.
For every new item $i_\ell$, we have
$\hr_{i_\ell}\geq\frac{r_i}{n\tr_i/\sg}\geq\frac{32\sg}{33n}\geq\frac{1}{2n'}$
and $\hr_{i_\ell}\leq\frac{1}{K}\cdot\frac{3}{2}\cdot\frac{r_i}{n\tr_i/\sg}
\leq\frac{384\sg}{217n}\leq\frac{2}{n'}$. Thus, $(w,\hP)$ is isotropic.

Now consider the width of $(w,\hP)$. 
For $t=1,\ldots,k$, define $p'^t\in\R^n$ to be the vector where $p'^t_i=0$ if $i\in S$,
and $p'^t_i=p^t_i$ otherwise.
For any distinct $t, t'\in[k]$, we have
$\|\hp^t-\hp^{t'}\|_2\geq\frac{\|\hp^t-\hp^{t'}\|_1}{\sqrt{n'}}$ and
$$
\|\hp^t-\hp^{t'}\|_1=\frac{\|p'^t-p'^{t'}\|_1}{K}
\geq\|p^t-p^{t'}\|_1-\sum_{i\in S}\max\{p^t_i,p^{t'}_i\}
\geq\zeta-n\cdot\frac{4\sg}{w_{\min}n}\geq\zeta/2.
$$
Let $\hA=\sum_{t=1}^k w_t(\hp^t-\hr)(\hp^t-\hr)^\dagger$, which is an $n'\times n'$
matrix. We need to prove that the smallest non-zero eigenvalue of $\hA$ is at least
$\frac{\zeta^2}{4}\cdot\|\hr\|_\infty$. It will be convenient to define the following matrices.
Let $B\in\R^{([n]\sm S)\times([n]\sm S)}$ be the matrix defined by setting
$B_{i,j}=A_{i,j}$ for all $i,j\notin S$. 
Define $A'$ to be the $n\times n$ matrix obtained by padding $B$ with 0s: set
$A'_{i,j}=A_{i,j}=B_{i,j}$ if $i,j\notin S$, and equal to 0 otherwise. 
It is easy to see that the non-zero eigenvalues of $A'$ coincide with
the non-zero eigenvalues of $B$.
Define $X\in\R^{n'\times([n]\sm S)}$ as follows. Letting 
$\{i_\ell\}_{i\notin S,\ell=1,\ldots,n_i}$ index the rows of $X$, we set   
$X_{i_\ell,j}=\frac{1}{Kn_i}$ if $j=i$, and $0$ otherwise. 
Notice that $\hA=XBX^\dagger$. To see this, it is convenient to define a padded version
$Y\in\R^{n'\times [n]}$ of $X$ by setting $Y_{i_\ell,j}=X_{i_\ell,j}$ if $j\notin S$ and $0$
otherwise. Then, we have $\hp^t=Yp^t$ for all $t\in[k]$, and hence, 
$\hA=YAY^\dagger=XBX^\dagger$. 

Note that $\rank(A')\leq\rank(A)=k'$. 
Consider $A-A'$. Suppose $i\in S$, so $p^t_i\leq\frac{4\sg}{w_{\min}n}$ for all $t\in[k]$. 
Then, 
$$
|(A-A')_{i,j}|=|A_{i,j}|=|M_{i,j}-R_{i,j}|\leq
\max\Bigl\{\sum_{t=1}^kw_tp^t_ip^t_j,r_ir_j\Bigr\}\leq\frac{4\sg}{w_{\min}n}\cdot r_j
\leq\frac{4\sg}{w_{\min}n}\cdot\|r\|_\infty.
$$ 
Hence, $\|A-A'\|_\op\leq\|A-A'\|_F\leq\frac{8\sg}{w_{\min}}\cdot\|r\|_\infty$. 
By Lemma~\ref{evper}, this implies that 
$$
\ld_{k'}(B)=\ld_{k'}(A')\geq\ld_{k'}(A)-\|A-A'\|_\op
\geq\Bigl(\zeta^2-\frac{8\sg}{w_{\min}}\Bigr)\|r\|_\infty
\geq\frac{\zeta^2}{2}\cdot\|r\|_\infty.
$$
We now argue that $\ld_{k'}(XBX^\dagger)\geq\ld_{k'}(B)/(\max_i n_i)$. By the
Courant-Fischer theorem (see, e.g., Theorem 4.2.11 in~\cite{HornJ85}), this is equivalent
to showing that there exist vectors $y^1,\ldots,y^{k'}\in\R^{n'}$, such that for every
unit vector $v\in\Span(y^1,\ldots,y^{k'})$, we have
$v^\dagger(XBX^\dagger)v\geq\frac{\ld_{k'}(B)}{\max_i n_i}$.  
We know that there are vectors $u^1,\ldots,u^{k'}\in\R^{[n]\sm S}$ such that
$zBz^\dagger\geq\ld_{k'}(B)\|z\|_2$ for every $z\in\Span(u^1,\ldots,u^{k'})$. Set
$y^t_{i_\ell}=u^t_i$ for every copy $i_\ell$ of item $i\in[n]\sm S$, and every
$t\in[k']$. Consider any $v\in\Span(y^1,\ldots,y^{k'})$. We have that 
$z=X^\dagger v\in\Span(u^1,\ldots,u^{k'})$, and since $v_{i_\ell}=z_i$ for every copy
$i_\ell$ of item $i\in[n]\sm S$ we have that $\|v\|_2^2\leq(\max_i n_i)\|z\|_2^2$.
Therefore, if $v$ is a unit vector, we have 
$v^\dagger XBX^\dagger v=z^\dagger Bz\geq\ld_{k'}(B)\|z\|_2^2
\geq\frac{\ld_{k'}(B)}{\max_i n_i}$. 

Putting everything together, we have that
$\ld_{k'}(\hA)\geq\frac{\zeta^2\|r\|_\infty}{2\max_i n_i}$. 
Note that $\|r\|_\infty\geq\frac{32}{33}\|\tr\|_\infty$ and 
$\frac{\|\tr\|_\infty}{\max_i n_i}\geq\frac{\sg}{n}\geq\frac{217}{384}\|\hr\|_\infty$. 
So the smallest non-zero eigenvalue of $\hA$ is
$\ld_{k'}(\hA)\geq\frac{\zeta^2}{4}\|\hr\|_\infty$. 
\end{proof}

\vspace{-1ex}
\paragraph{Algorithm overview.}
Our algorithm for learning an isotropic $k$-mixture source on $[n]$ takes three
parameters: $\zeta\leq 1$ such that $(w,P)$ is $\zeta$-wide,
$\omega\in\NN$, which controls the success
probability of the algorithm, and $\delta\in (0,1)$, which controls
the statistical distance between the constituents of the learnt
model and the constituents of the correct model.
For convenience, we assume that $\dt$ is sufficiently small.
The output of the algorithm is a $k$-mixture source $(\tilde w,\tilde P)$ such that with
probability $1-O\bigl(\frac{1}{\w}\bigr)$, $\|w-\tilde w\|_\infty$ and
$\|p^t-\tilde p^t\|_1$ for all $t\in[k]$ tend to 0 as $\dt\to 0$
(see Theorem~\ref{thm: main}).

The algorithm (see Algorithm~\ref{mainalg}) consists of three stages.
First, we reduce the dimensionality of the problem from $n$ to $k'$ using only
1- and 2-snapshots. By Lemma~\ref{rest}, we have an
estimate $\tr$ that is component-wise close to $r$. Thus, $\tR=\tr\tr^\dagger$ is close in
operator norm to $R$.
So we focus on learning the column space of $A$ for which we employ spectral techniques.
Leveraging Theorem~\ref{rmat}, we argue (Lemma~\ref{Aest}) that by using
$O(n\ln^6 n)$ 2-snapshots, one can compute (with high probability)
a good enough estimate $\tilde M$ of $M$,
and hence obtain a PSD matrix $\tilde A$ such that $\|A-\tA\|_\op$ is
small. 

The remaining task is to learn the projection of $P$ on the affine space $\tr+\col(\tA)$,
and the mixture weights, which then
yields the desired $k$-mixture source $(\tw,\tP)$.
We divide this into two steps. 
We choose a random orthonormal basis $\{b_1,\ldots,b_{k'}\}$ of $\col(\tA)$.
For each $b_j$, we consider the projected $k$-mixture source $(w,\pi_{b_j}(P))$ on $\R$.
In Section~\ref{onedim}, we devise a procedure to learn the corresponding $k$-spike
distribution $(w,\E[\pi_{b_j}(P)])$ using $(2k-1)$-snapshots from $(w,\pi_{b_j}(P))$ (which
we can obtain using $(2k-1)$-snapshots from $(w,P)$).
Applying this procedure (see Lemma~\ref{blearn}), we obtain weights
$\tw^j_1,\ldots,\tw^j_k$ and $k$ (distinct) values $\al^j_1,\ldots,\al^j_k$ such
that each true spike $(w_t,b_j^\dagger p^t)$ maps to a distinct inferred spike
$(\tw^j_{\sg^j(t)},\al^j_{\sg^j(t)})$.

Finally, we match up $\sg_j$ and $\sg_{k'}$ 
for all $j\in[k'-1]$ to obtain $k$ points in
$\tr+\col(\tA)$ that are close to the projection of $P$ on $\tr+\col(\tA)$.
For every $j\in[k'-1]$, we generate a random unit ``test vector''
$z_j$ in $\Span(b_j,b_{k'})$ and learn the projections $\{z_j^\dagger p^t\}_{t\in[k]}$.
Since $(w,P)$ is $\zeta$-wide, results about random projections and
the guarantees obtained from our $k$-spike learning procedure imply that 
$z_j^\dagger(\al^j_{t_1}b_j+\al^{k'}_{t_2}b_{k'})$ is close to some value
in $\{z_j^\dagger p^t\}_{t\in[k]}$ iff there is some $t$ such that $\al^j_{t_1}$ and
$\al^{k'}_{t_2}$ are close respectively to $b_j^\dagger p_t$ and $b_{k'}^\dagger p^t$
(Lemma~\ref{reconcile}). 
Thus, we can use the learned projections 
{of $\{z_j^\dagger p^t\}_{t\in[k]}$ to match
up $\{\al^j_t\}_{t\in[k]}$ and $\{\al^{k'}_t\}_{t\in[k]}$.}

{\small \vspace{10pt} 
\begin{algorithm} \label{mainalg}
\vspace{-5pt} \hrule \vspace{5pt}

\noindent
Input: an isotropic $\zeta$-wide $k$-mixture source $(w,P)$ on $[n]$, and parameters
$\w>1$ and $\dt>0$.

\noindent
Output: a $k$-mixture source $(\tw,\tP)$ on $[n]$ that is ``close'' to $(w,P)$.

\noindent
Define $T=3\w k^4$, $H=\frac{4}{w_{\min}^2\zeta\sqrt{n}}$ and
$L=\frac{\zeta}{64\w^{1.5}k^4\sqrt{n}}$.
We assume that
$\dt\leq\frac{w_{\min}^3\zeta^4}{2^{29}\w^5 k^{16}}$. 
Let $\kp=\kp(2+\ln\w)$ be given by Theorem~\ref{rmat}; we assume $\kp\geq 1$ for
convenience.
Define $c=\frac{6400\kp^2}{w_{\min}^2\dt^2}\cdot\ln\bigl(\frac{1}{\dt}\bigr)$.
We assume that $w_{\min}^2\geq\frac{240\kp\ln^{2.5} n}{\sqrt{n}}$.
\end{algorithm}
\begin{labellist}
\item {\bf Dimension reduction.\ }
\begin{list}{A\arabic{enumi}.\arabic{enumii}}{\topsep=0.5ex \itemsep=0ex \usecounter{enumii}}
\item Use Lemma~\ref{rest} with $\mu=2+\ln\w$ and $\sg=\frac{\dt}{48}$ to compute an
estimate $\tr$ of $r$. Set $\tR=\tr\tr^\dagger$.
\item Independent of all other random variables, choose a Poisson random variable
$N_2$ with expectation $\E[N_2]=cn\ln^6 n$.
Choose $N_2$ independent 2-snapshots and construct a symmetric $n\times n$
matrix $\tilde{M}$ as follows: set $\tilde{M}_{i,i}=$ frequency of the 2-snapshot
$(i,i)$ in the sample for all $i\in[n]$, and $\tilde{M}_{i,j} = \tilde{M}_{j,i}=$ half the
combined frequency of 2-snapshots $(i,j)$ and $(j,i)$ in the sample, for all $i,j\in [n],
i\ne j$.
\item Compute the spectral decomposition $\tM-\tR=\sum_{i=1}^n\ld_iv_iv_i^\dagger$, 
where $\lambda_1\ge\ldots\ge\lambda_n$.

\item Set $\tA=\sum_{i:\ld_i\geq\zeta^2/2n}\ld_i v_iv_i^\dagger$. Note that $\tA$ is PSD.
\end{list}

\item {\bf Learning projections of \boldmath $(w,P)$ on random vectors in $\col(\tA)$.\ }
\begin{list}{A\arabic{enumi}.\arabic{enumii}}{\topsep=0.5ex \itemsep=0ex \usecounter{enumii}}
\item Pick an orthonormal basis $B=\{b_1,\ldots,b_{k'}\}$ for $\col(\tA)$ uniformly at
random.
\item Set $(\tw^j,\al^j)\assign\learn\bigl(b_j,\dt,\frac{1}{6\w k}\bigr)$
for all $j=1,\ldots,k'$.
\end{list}

\item {\bf Combining the projections to obtain \boldmath $(\tw,\tP)$.}
\begin{list}{A\arabic{enumi}.\arabic{enumii}}{\topsep=0.5ex \itemsep=0ex \usecounter{enumii}}
\item Pick $\tht\in[0,2\pi]$ uniformly at random.
\item For each $j=1,\ldots,k'-1$, we do the following.
\begin{list}{--}{\topsep=0ex \itemsep=0ex \addtolength{\leftmargin}{-1ex}}
\item Let $z_j=b_j\cos\tht+b_{k'}\sin\tht$.
\item Set $(\hw^j,\hal^j)\assign\learn\bigl(z_j,\dt,\frac{1}{6\w k}\bigr)$.
\item For each $t_1,t_2\in[k]$, if there exists $t\in[k]$ such that
$\bigl|(\al^j_{t_1}b_j+\al^{k'}_{t_2}b_{k'})^\dagger z_j-\hal^j_t\bigr|\leq(\sqrt{2}+1)L/(2+5T)$
then set $\vro^j(t_2)=t_1$.
\end{list}
\item Define $\vro^{k'}(t)=t$ for all $t\in[k]$.
\item For every $t\in[k]$: set $\tw_t=\bigl(\sum_{j=1}^{k'}\tw^j_{\vro^j(t)}\bigr)/k'$, 
$\hp^t=\tr+\sum_{j=1}^{k'}\bigl(\al^j_{\vro^j(t)}-b_j^\dagger\tr\bigr)b_j$, and
$\tp^t=\arg\min_{x\in\Dt^{n-1}}\|x-\hp^t\|_1$, 
which can be computed by solving an LP.
Return $\bigl(\tw,\tP=(\tp^1,\ldots,\tp^k)\bigr)$.
\end{list}
\end{labellist}

\begin{labelalg}{$\learn(v,\procacc,\procerrp)$}  

\noindent
Input: a unit vector $v\in\col(\tA)$, and parameters $\procacc>0$, $\procerrp>0$. We
assume that (a) $|v^\dagger(p-q)|\geq L$ for all distinct $p,q\in P$;
and (b) $1024k\procacc<\frac{w_{\min}L}{16H}$.

\noindent
Output: a $k$-spike distribution $\bigl(\bw,(\gm_1,\ldots,\gm_k)\bigr)$ close to
$(w,\E[\pi_v(P)])$.
\end{labelalg}
\vspace{-8pt}
\begin{labellist}[L]
\item Solve the following convex program:
\begin{equation}
\min \quad \|x\|_\infty \quad \text{s.t.} \quad
v^\dagger x\geq 1-\frac{4\dt}{\zeta^2}, \quad \|x\|_2^2\leq 1
\tag{Q$_v$} \label{ajlp}
\end{equation}
to obtain $x^*$; set $a=\frac{x^*}{\|x^*\|_2}$.
We prove in Lemma~\ref{helper} that 
$\|a\|_\infty\leq H$ and 
$|a^\dagger(p-q)|\geq\frac{L}{2}$
for all $p,q\in P,\ p\neq q$.

\item Let $s=\procacc^{4k}$.
Apply the procedure in Section~\ref{onedim} leading to Theorem~\ref{onedthm} 
for $\bigl(w,\pi_{a/2H}(P)\bigr)$
to infer a $k$-spike distribution $(\bw,\beta)$ that, with
probability at least $1-\procerrp$, is within transportation distance
$O\bigl(s^{\Omega(1/k)}\bigr)$ from $\bigl(w,\E[\pi_{a/2H}(P)]\bigr)$.
This uses a sample of $(2k-1)$-snapshots of size
$3k2^{4k}s^{-4k}\ln(4k/\procerrp)$.

\item For every $t\in[k]$, set $\gm_t=(2H\beta_t)(a^\dagger v)$.
Return $(\bw,\gm)$.
\end{labellist}

\begin{remark} \label{algremk}
We cannot compute the spectral decomposition in step A1.3 exactly, or solve \eqref{ajlp} 
exactly in step L1, since the output may be irrational. However, one can obtain a
decomposition such that  
$\|\tM-\tR-\sum_{i=1}^n\ld_iv_iv_i^\dagger\|_\op=O\bigl(\frac{\dt}{n}\bigr)$ and compute
a 2-approximate solution to \eqref{ajlp} in polytime, and this suffices: slightly
modifying the constants $H$ and $c$ makes the entire analysis go through. We have chosen
the presentation above to keep exposition simple.
\end{remark}
\hrule
}

\vspace{-0.5ex}
\section{Analysis} \label{analysis} \label{ANALYSIS}

\begin{theorem} \label{mainthm} \label{thm: main}
Algorithm~\ref{mainalg} uses
$O\bigl(\frac{\ln \w}{\dt^3}\cdot n\ln n\bigr)$ 1-snapshots,
$O\bigl(\frac{\ln^2 \w\ln(1/\dt)}{\dt^2w_{\min}^2}\cdot n\ln^6 n\bigr)$ 2-snapshots, and
$O\bigl(\frac{k2^{4k}}{\dt^{16k^2}}\cdot\ln(24\w k^2)\bigr)$ $(2k-1)$-snapshots, and
computes a $k$-mixture source $(\tw,\tP)$ on $[n]$ such that with probability
$1-O\bigl(\frac{1}{\w}\bigr)$, there is a permutation $\sg:[k]\mapsto[k]$ such that
for all $t=1,\ldots,k$,
$$
|w_t-\tw_{\sg(t)}|= O\Bigl(\frac{\dt\w^{1.5}k^5}{w_{\min}^2\zeta^2}\Bigr) \qquad \text{and} 
\qquad \|p^t-\tp^{\sg(t)}\|_1= O\Bigl(\frac{\sqrt{k\dt}}{w_{\min}^{1.5}\zeta}\Bigr).
$$
Hence,
$\displaystyle \Tran(w,P; \tw,\tP) = O\Bigl(\frac{\sqrt{k\dt}}{w_{\min}^{1.5}\zeta}\Bigr)$.
The running time is polynomial in the sample size.
\end{theorem}

The roadmap of the proof is as follows.
By Lemma~\ref{rest}, with probability at least $1-\frac{1}{\w n^2}$,
$\bigl(1-\frac{\dt}{48}\bigr)r_i\leq\tr_i\leq\bigl(1+\frac{\dt}{48}\bigr)r_i$ for all
$i\in[n]$. We assume that this holds in the sequel.
In Lemma~\ref{Aest}, we prove that the matrix
$\tA$ computed after step A1 is a good estimate of $A$. In Lemma~\ref{alldir}, we derive some
properties of the column space of $A$.
Lemma~\ref{helper} then uses these properties to show that algorithm $\learn$ returns a
good approximation to $(w,\E[\pi_v(P)])$.
Claim~\ref{randprojn} and Lemma~\ref{jlprojn} prove that the projections of the
mixture constituents on the $b_j$s and the $z_j$s are well-separated.
Combining this with Lemma~\ref{helper}, we prove in Lemma~\ref{blearn} that with suitably
large probability, every true spike $(w_t,b_j^\dagger p^t)$ maps to a distinct nearby
inferred spike on every $b_j,\ j\in[k']$, and similarly every true spike
$(w_t,z_j^\dagger p^t)$ maps to a distinct nearby inferred spike on every
$z_j,\ j\in[k'-1]$.
Lemma~\ref{reconcile} shows that one can then match up the spikes on the different
$b_j$s. This yields $k$ points in $\col(\tA)$ that are close to the projection of $P$ on
$\col(\tA)$. Finally, we argue that this can be mapped to a $k$-mixture source $(\tw,\tP)$
that is close to $(w,P)$.

\begin{lemma}\label{lm: A estimation} \label{Aest}
With probability at least $1-\frac{1}{n\omega}$, the matrix $\tA$ computed after step A1
satisfies $\rank(\tA)=k'=\rank(A)$ and $\|A - \tilde{A}\|_\op\le\frac{\delta}{n}$.
\end{lemma}

\begin{proof}
Recall that $k'=\rank(A)$.
Let $B=\tM-\tR=\sum_{i=1}^n \ld_i v_iv_i^\dagger$, where $\ld_1\geq\ldots\geq\ld_n$. We
prove below that with probability at least $1-\frac{1}{n\w}$, we have
$\|M-\tM\|_\op\leq\frac{\dt}{4n}$ and $\|R-\tR\|_\op\leq\frac{\dt}{4n}$.
This implies that $\|A-B\|_\op\leq\|M-\tM\|_\op+\|R-\tR\|_\op\leq\frac{\dt}{2n}$.
Hence, by Lemma~\ref{evper}, it follows that by the $\zeta$-wide assumption,
$\ld_{k'}\geq\frac{\zeta^2}{n}-\frac{\dt}{2n}\geq\frac{3\zeta^2}{4n}$, and
$|\ld_i|\leq\frac{\dt}{2n}\leq\frac{\zeta^2}{4n}$ for all $i>k'$.
Thus, we include exactly $k'$ eigenvectors when defining $\tA$, so $\rank(\tA)=k'$.
Since $\tA$ is the closest rank-$k'$ approximation in operator norm to $B$, we have
$\|A - \tilde{A}\|_\op\leq\|A-B\|_\op+\|B-\tA\|_\op\leq 2\|A-B\|_\op\leq\frac{\dt}{n}$.

It is easy to see that $|\tR_{i,j}-R_{i,j}|\leq 3\sg r_{i,j}$, where $\sg=\dt/48$,
and so $\|R-\tR\|_\op\leq\|R-\tR\|_F\leq\frac{\dt}{4n}$.
Bounding $\|M-\tM\|_\op$ is more challenging. 
We carefully define a matrix whose entries
are independent random variables with bounded variance, and then apply Theorem~\ref{rmat}. 

Note that $M_{i,j}\leq\min\bigl\{\frac{2}{n},\frac{4}{w_{\min}n^2}\bigr\}$
due to isotropy.
Let $K=\frac{4\ln(1/\dt)}{\dt}$ and $K'=\frac{5\ln(1/\dt)}{\dt}$.
Let $D = N_2\cdot\bigl(\tilde{M} - M\bigr)$.
Let $X^\ell_{i,i}=1$ if the $\ell$-th snapshot is $(i,i)$, for $i\in[n]$, and for
$i,j\in[n],i\neq j$, let $X^\ell_{i,j}=X^\ell_{j,i}=\frac{1}{2}$ if the $\ell$-th
2-snapshot is $(i,j)$ or $(j,i)$, and 0 otherwise. Let
$Y^\ell_{i,j}=X^\ell_{i,j}-M_{i,j}=X^\ell_{i,j}-\E[X^\ell_{i,j}]$;
so $D_{i,j}=\sum_{\ell=1}^{N_2}Y^\ell_{i,j}$ for all $i,j\in[n]$.
We have
$\sg^2(n_2):=\Var[D_{i,j}|N_2=n_2]=n_2\Var[X^1_{i,j}]\leq n_2\E[(X^1_{i,j})^2]\leq n_2M_{i,j}$.
For $n_2\leq 2cn\ln^6 n$, we have
$\sg^2(n_2)\leq\frac{8c\ln^6 n}{w_{\min}^2 n}\leq\frac{\ln n\ln(1/\dt)}{\dt^2}$
(since $w_{\min}^4\geq\frac{57600\kp^2\ln^5n}{n}$).
So by Bernstein's inequality (Lemma~\ref{bernstein}),
\begin{align*}
\Pr[|D_{i,j}|>K\ln n|N_2=n_2] & \leq 
2\exp\Bigl(-\frac{K^2\ln^2 n}{2\bigl(\sg^2(n_2)+K\ln n/3\bigr)}\Bigr) \\
& \leq 2\max\Bigl\{\exp\bigl(-\tfrac{K^2\ln^2 n}{4\sg^2(n_2)}\bigr),
\exp\bigl(-\tfrac{3K\ln n}{4}\bigr)\Bigr\}
\leq \frac{2\dt}{n^3}.
\end{align*}
Since $\Pr[N_2>2c\ln^6 n]\leq n^{-3}$, we can say that with probability at least
$1-2n^{-2}$, we have $|D_{i,j}|\le K\ln n$ for every $i,j\in [n]$ and $N_2\leq 2c\ln^6 n$.

Define a matrix $D'$ by putting, for every $i,j\in [n]$,
$D'_{i,j} = \sign(D_{i,j})\cdot \min\bigl\{|D_{i,j}|,K\ln n\bigr\}$.
Put $D'' = D' - \E[D']$.
Clearly, $\E[D''_{i,j}] = 0$ for every $i,j\in [n]$. We prove below that
$\bigl|\E[D'_{i,j}]\bigr|\leq\frac{3\dt c\ln^6 n}{n^2}\leq\frac{\ln n\ln(1/\dt)}{\dt}$;
therefore, $|D''_{i,j}|\leq K'\ln n$. 
The entries of $D$ are independent random variables as $N_2$ is a Poisson random variable;
hence, the entries of $D''$ are also independent random variables.
Also $\Var[D''_{i,j}]\leq\Var[D_{i,j}]$ since censoring a random variable to an
interval can only reduce the variance.
Note that $D_{i,j}=\sum_{\ell=1}^{N_2}Y^\ell_{i,j}$ follows the compound Poisson
distribution. So we have
$$
\Var[D_{i,j}]=\E[N_2]\cdot\E[(Y^1_{i,j})^2]=\E[N_2]\cdot\Var[X^1_{i,j}] 
\leq\E[N_2]M_{i,j}
\leq\frac{4c\ln^6 n}{w_{\min}^2\cdot n}
\leq\frac{\hc^2K'^2\ln^6 n}{n}
$$
where $\hc=\max\bigl\{\frac{2\sqrt{c}}{w_{\min}K'},\kp^2\bigr\}$.
Thus, by Theorem~\ref{rmat}, the constant $\kp=\kp(2+\ln\w)>0$ is such that with probability
at least $1 - \frac{1}{n^2\omega}$
\begin{equation}
\|D''\|_\op \le 2\cdot\frac{\hc K'\ln^3 n}{\sqrt{n}}\cdot\sqrt{n} +
\kp\sqrt{K'\ln n\cdot\frac{\hc K'\ln^3 n}{\sqrt{n}}}
\cdot \sqrt[4]{n}\cdot\ln n 
\leq\bigl(2K'\hc+\kp K'\sqrt{\hc}\bigr)\ln^3 n.
\label{dopbnd}
\end{equation}
We have
$\Pr\bigl[N_2\ge\frac 1 2 \E[N_2]\bigr]\geq 1 - n^{-2}$,
Thus, with probability at least $1 - \frac{1}{n\omega}$, we have that
$N_2\ge\frac 1 2 \E[N_2]$, $D' = D$, and $\|D''\|_\op$ is bounded by \eqref{dopbnd}.
We show below that
$2\|\E[D']\|_\op/\E[N_2]\leq 6\dt n^{-2}\leq\dt/20n$.
One can verify that $4K'\hc/c\leq\dt/10$ and $2\kp K'\sqrt{\hc}/c\leq\dt/10$.
Therefore, with probability at least $1 - \frac{1}{n\omega}$, we have that
$\|M - \tilde{M}\|_\op = \frac{1}{N_2}\cdot\|D\|_\op
\le \frac{2}{\E[N_2]}\cdot\left(\|D''\|_\op + \|\E[D']\|_\op\right)
\le \frac{\delta}{4n}$.

Finally, we bound $\|\E[D']\|_\op$. We have
$\|\E[D']\|_\op\leq\|\E[D']\|_F\leq n\cdot\max_{i,j}\bigl|\E[D'_{i,j}]\bigr|$.
Let $\mu=cn\ln^6 n=\E[N_2]$. 
Fix any $i,j$. 
We have $\bigl|\E[D'_{i,j}]\bigr|
=\bigl|\E[D'_{i,j}-D_{i,j}]\bigr|\leq\E\bigl[|D'_{i,j}-D_{i,j}|\bigr]$.
For any $n_2\leq 2\ln(1/\dt)\mu$, we have
$\Var[D_{i,j}|N_2=n_2]\leq n_2M_{i,j}
\leq\frac{8c\ln(1/\dt)\ln^6 n}{w_{\min}^2 n}$.
So by Bernstein's inequality, we have that
$\Pr[|D_{i,j}|>K\ln n|N_2\leq 2\ln(1/\dt)\mu]<2\dt n^{-3}$.
Also, $|D'_{i,j}-D_{i,j}|\leq N_2$ always.
Therefore,
$$
\E\bigl[|D'_{i,j}-D_{i,j}|\bigl|N_2=n_2\bigr]\leq
\begin{cases} 2\dt n^{-3}n_2 & \text{if $n_2\leq 2\ln\bigl(\tfrac{1}{\dt}\bigr)\mu$}; \\
n_2 & \text{otherwise}
\end{cases}
$$
and
$\E\bigl[|D'_{i,j}-D_{i,j}|\bigr]
\leq\mu-\Pr[N_2\leq 2\ln(1/\dt)\mu]\E[N_2|N_2\leq 2\ln(1/\dt)\mu](1-2\dt n^{-3})$.
Since $N_2$ is Poisson distributed, we have
\begin{equation*}
\begin{split}
\Pr[N_2\leq 2\ln(1/\dt)\mu]\E[N_2|N_2\leq 2\ln(1/\dt)\mu]\ 
& = \sum_{\ell=0}^{\lfloor 2\ln(1/\dt)\mu\rfloor}\ell\cdot\frac{\mu^\ell e^{-\mu}}{\ell !} 
\ =\ \mu\sum_{\ell=0}^{\lfloor 2\ln(1/\dt)\mu\rfloor-1}\frac{\mu^\ell e^{-\mu}}{\ell !} \\
& \geq\ \mu\Pr[N_2\leq\ln(1/\dt)\mu]\ \geq\ \mu(1-\dt n^{-3}).
\end{split}
\end{equation*}
Thus, $\E\bigl[|D'_{i,j}-D_{i,j}|\bigr]\leq\mu-\mu(1-\dt n^{-3})(1-2\dt n^{-3})\leq 3\dt n^{-3}\mu$,
and $2\|\E[D']\|_\op/\E[N_2]\leq 6\dt n^{-2}$.
\end{proof}

We assume in the sequel that the high-probability event stated in Lemma~\ref{Aest}
happens. Thus, Lemma~\ref{prjerr} implies that
$\|\Pi_A-\Pi_{\tA}\|_\op\leq\frac{2\sqrt{\dt}}{\zeta}$.

\begin{lemma} \label{lm: all directions} \label{alldir}
For every unit vector $b\in\Span(A)$, $\|b\|_\infty\le \frac{2}{w_{\min}^2\zeta\sqrt{n}}$.
\end{lemma}

\begin{proof}
Recall that $A=\sum_{t=1}^k w_t(p^t-r)(p^t-r)^\dagger$, and the smallest non-zero
eigenvalue of $A$ is at least $\zeta^2/n$. Note that
$\col(A)=\Span\{p^1-r,\ldots,p^k-r\}$. Let $Z=\conv(P)$.
If $r+b\in Z$, then $\|r+b\|_\infty\leq\frac{2}{w_{\min} n}$, $r+b\geq 0$, and
$\|r\|_\infty\leq\frac{2}{n}$ imply that $\|b\|_\infty\leq\frac{2}{w_{\min}n}$.
Otherwise, let the line segment $[r,r+b]$ intersect
the boundary of $Z$ at some point $b'$.
We show that $\|r-b'\|^2_2\geq\frac{\zeta^2w_{\min}^2}{n}$. The lemma then follows since
$b=(b'-r)/\|b'-r\|_2$ and so
$\|b\|_\infty=\frac{\|b' - r\|_\infty}{\|b' - r\|_2}
\leq\frac{2}{w_{\min}^2\zeta\sqrt{n}}$.

Let $S$ be a facet of $Z$ such that $b'\in S,\ r\notin S$ (note that $r$ is in the strict
interior of $P$).
Since $Z\sse\col(A)$, 
one can find a unit vector $v\in\col(A)$ such that $S$ is exactly the set of points that
minimize $v^\dagger x$ over $x\in Z$.
Let $d_L=v^\dagger r-\min_{x\in Z}v^\dagger x=v^\dagger(r-b')$.
We lower bound $\|r-b'\|_2$ by $d_L$.
Note that $d_L>0$.
Clearly, $v^\dagger(p^t-r)\geq -d_L$ for all $t\in[k]$.
Projecting $P$ onto $v$, we have that
(a) $\sum_{t=1}^k w_t v^\dagger (p^t-r)=0$; and
(b) $v^TAv=\sum_{t=1}^k w_t \left(v^\dagger (p^t - r)\right)^2\ge\frac{\zeta^2}{n}$ since
$v\in\col(A)$ and $(w,P)$ is $\zeta$-wide.
Let $W_L = \sum_{t: v^\dagger(p^t - r)\leq 0}w_t$,
let $W_R = 1 - W_L\geq w_{\min}$, and let $d_R = \max_t\{v^\dagger (p^t - r)\}$.
Then, $0=\sum_{t=1}^k w_tv^\dagger(p^t-r)\geq W_L(-d_L)+w_{\min}d_R$, so
$d_R\le d_L\cdot\frac{W_L}{w_{\min}}$,
Also $\frac{\zeta^2}{n}\le\sum_{t=1}^k w_t\left(v^\dagger(p^t-r)\right)^2\leq
W_L\cdot d_L^2 + W_R\cdot d_R^2\le
W_L\cdot d_L^2 +W_R\cdot d_L^2\cdot\frac{W_L^2}{w_{\min}^2}\le
\frac{d_L^2}{w_{\min}^2}$. So, $d_L^2\ge \frac{\zeta^2 w_{\min}^2}{n}$.
\end{proof}

\begin{lemma} \label{helper}
If the assumptions stated in Algorithm $\learn$ are satisfied, then:
(i) the vector $a$ computed in $\learn$ satisfies $\|a\|_\infty\leq H$, and
$|a^\dagger(p-q)|\geq L/2$ for every two mixture constituents $p,q\in P$; \linebreak
(ii) with probability at least $1-\procerrp$, the output
$(\bw,\gm)$ of $\learn$ satisfies the following: there is a permutation
$\sg:[k]\mapsto[k]$ such that for all $t=1,\ldots,k$,
$$
|w_t-\bw_{\sg(t)}| =O\Bigl(\frac{\procacc \w^{1.5}k^5}{w_{\min}^2\zeta^2}\Bigr), 
\qquad
|v^\dagger p^t-\gm_{\sg(t)}|
\leq\frac{2048kH\procacc}{w_{\min}}+\frac{8\sqrt{2\dt}}{w_{\min}\zeta\sqrt{n}}
\leq\frac{2048kH\procacc}{w_{\min}}+\frac{L}{8T}.
$$
\end{lemma}

\begin{proof}
We have
$v^\dagger\Pi_A(v)=1-\|v-\Pi_A(v)\|_2^2
=1-\|(\Pi_{\tA}-\Pi_A)v\|^2_2\geq 1-\frac{4\dt}{\zeta^2}$.
Thus, $\Pi_A(v)$ is feasible to \eqref{ajlp}, and since $\|\Pi_A(v)\|_2\leq 1$, by
Lemma~\ref{alldir}, the optimal solution $x^*$ to \eqref{ajlp} satisfies
$\|x^*\|_\infty\leq 2\|\Pi_A(v)\|_\infty\leq H/2$.
Also $\|x^*\|^2_2\geq v^\dagger x^*\geq 1-\frac{4\dt}{\zeta^2}\geq\frac{1}{4}$, so
$\|a\|_\infty\leq H$.
Note that
$\|v-a\|_2^2=2(1-v^\dagger a)\leq 2(1-v^\dagger x^*)\leq\frac{8\dt}{\zeta^2}$.
It follows that for any two mixture constituents $p, q$, we have
\begin{align*}
|a^\dagger (p-q)| \geq |v^\dagger(p-q)|-|(v-a)^\dagger(p-q)|
& \geq |v^\dagger(p-q)|-\frac{2\sqrt{2\dt}}{\zeta}\|p-q\|_2 \\
& \geq|v^\dagger(p-q)|-\frac{8\sqrt{2\dt}}{w_{\min}\zeta\sqrt{n}} 
\geq |v^\dagger(p-q)|-\frac{L}{2}\geq \frac{L}{2}.
\end{align*}
This proves part (i). 
For part (ii), we note that any two spikes in the $k$-spike mixture
$\bigl(w,\E[\pi_{a/2H}(P)]\bigr)$ are 
separated by a distance of at least $L/4H$. Since $s<L/4H$,
Theorem~\ref{onedthm} guarantees that with a sample of $(2k-1)$-snapshots of size
$3k2^{4k}s^{-4k}\log(4k/\procerrp)$, with probability at least $1-\procerrp$, the learned
$k$-spike distribution $(\bw,\beta)$ satisfies
$\Tran\bigl(w,\E[\pi_{a/2H}(P)]; \bw,\beta\bigr)\leq 1024ks^{1/(4k)}=1024k\procacc<\frac{L w_{\min}}{8H}$.
Notice that this implies that there is a permutation $\sg:[k]\mapsto[k]$ such that $\forall t=1,\ldots,k$:
\begin{equation}
|(a/2H)^\dagger p^t-\beta_{\sg(t)}| \leq \frac{1024k\procacc}{w_{\min}} < \frac{L}{8H},  
\qquad 
|w_t-\bw_{\sg(t)}| =O\Bigl(\frac{k\procacc}{L/8H}\Bigr)
= O\Bigl(\frac{\procacc \w^{1.5}k^5}{w_{\min}^2\zeta^2}\Bigr).
\label{projnbnd} 
\end{equation}

Fix some $t\in[k]$. Let $t'=\sg(t)$.
From \eqref{projnbnd}, we know that
$|a^\dagger p^t-2H\cdot\beta_{t'}|=\frac{2048kH\procacc}{w_{\min}}$.
We bound $|v^\dagger p^t-a^\dagger p^t|$ and $|2H\beta_{t'}-\gm_{t'}|$, which
together with the above will complete the proof of the lemma.
We have $|(v-a)^\dagger p^t|\leq\|v-a\|_2\|p^t\|_2
\leq\frac{4\sqrt{2\dt}}{w_{\min}\zeta\sqrt{n}}$.
Since $\gm_{t'}=(2H\beta_{t'})a^\dagger v$ and $|\beta_{t'}|\leq\frac{1}{2}$, we have
$|2H\beta_{t'}-\gm_{t'}|\leq \frac{H\cdot 4\dt}{\zeta^2}
\leq\frac{16\dt}{w_{\min}^2\zeta^3\sqrt{n}}$.
It follows that $|v^\dagger p^t-\gm_{t'}|
\leq\frac{2048kH\procacc}{w_{\min}}+\frac{8\sqrt{2\dt}}{w_{\min}\zeta\sqrt{n}}
\leq\frac{2048kH\procacc}{w_{\min}}+\frac{L}{8T}$.
\end{proof}

\begin{claim} \label{randprojn}
Let $Z$ be a random unit vector in $\col(\tA)$ and $v\in\col(\tA)$.
$\Pr\bigl[|Z^\dagger v|<\frac{\|v\|_2}{32\w^{1.5}k^4}\bigr]<\frac{1}{3\w k'k^2}$.
\end{claim}

\begin{proof}
One way of choosing the random unit vector $Z$ is as follows. Fix an orthonormal
basis $\{u_1,\ldots,u_{k'}\}$ for $\col(\tA)$. We choose independent $N(0,1)$ random
variables $X_i$ for $i\in[k']$. Define $C=\sum_{i=1}^{k'}X_iu_i$ and set $Z=C/\|C\|_2$.
Set $a_1=\frac{\pi}{32\w^2k'^2k^4}$ and $a_2=2+\frac{4\ln(12\w k'k^2)}{k'}\leq 96\w k$.

Note that $C^\dagger v/\|v\|_2$ is distributed as $N(0,1)$. Therefore,
$\Pr\bigl[|C^\dagger v|\leq\|v\|_2\sqrt{a_1}\bigr]
\leq\sqrt{\frac{2a_1}{\pi}}\leq\frac{1}{4\w k'k^2}$.
Also, $\|C\|_2^2=\sum_{i=1}^{k'}X_{i}^2$ follows the $\chi_{k'}^2$ distribution.
So
$$
\Pr[\|C\|_2^2>a_2k'] < \Bigl(a_2e^{1-a_2}\Bigr)^{k'/2} 
< \exp\Bigl((1-a_2/2)k'/2\Bigr)
<\frac{1}{12\w k'k^2}.
$$
Observe that $\sqrt{\frac{a_1}{a_2k'}}\geq\frac{1}{32\w^{1.5}k^4}$. So if the ``bad''
event stated in the lemma happens, then $|C^\dagger v|\leq\|v\|_2\sqrt{a_1}$ or
$\|C\|_2^2\geq a_2k'$ happens; the probability of this is at most $\frac{1}{3\w k'k^2}$.
\end{proof}

\begin{lemma} \label{jlprojn}
With probability at least $1-\frac{1}{3\w}$, for every pair $p,q\in P$, we have
(i) $|b_j^\dagger(p-q)|\geq L$ for every $j\in[k']$ and
(ii) $|z_j^\dagger(p-q)|\geq L$ for every $j\in[k'-1]$. 
\end{lemma}

\begin{proof}
Define $\tp=\Pi_{\tA}(p)$ for a mixture constituent $p$. Clearly, for any $v\in\col(\tA)$,
$v^\dagger\tp=v^\dagger p$.
Recall that $\|\Pi_A-\Pi_{\tA}|\leq\frac{2\sqrt{\dt}}{\zeta}$.
So for every $p,q\in P$,
$\|\tp-\tq\|_2^2\geq\|p-q\|_2^2-\|(\Pi_A-\Pi_{\tA})(p-q)\|_2^2\geq\|p-q\|^2/4$;
hence, $\|\tp-\tq\|_2\geq\frac{\zeta}{2\sqrt{n}}$.
Notice that the $z_j$ vectors are also random unit vectors in $\col(\tA)$.
Applying Claim~\ref{randprojn} to each event involving one of the
$\{b_j\}_{j\in [k']}, \{z_j\}_{j\in[k'-1]}$ random unit vectors, and one of the
$\binom{k}{2}$ vectors $\|\tp-\tq\|$ for $\tp,\tq\in\Pi_{\tA}(P)$, and taking the union
bound over the at most $k'k^2$ such events completes the proof.
\end{proof}

\begin{lemma} \label{blearn}
With probability at least $1-\frac{2}{3\w}$, the $k$-spike distributions obtained in steps
A2 and A3 satisfy: 

\noindent
(i) For every $j\in[k']$, there is a permutation $\sg^j:[k]\mapsto[k]$ such that for all
$t\in[k]$, 
$$
|w_t-\tw^j_{\sg^j(t)}|=O\Bigl(\frac{\dt\w^{1.5}k^5}{w_{\min}^2\zeta^2}\Bigr), \qquad
|b_j^\dagger p^t-\al^j_{\sg^j(t)}|=O\Bigl(\frac{\sqrt{\dt}}{w_{\min}^{1.5}\zeta\sqrt{n}}\Bigr)
\ \ \text{and is at most}\ \frac{L}{2+5T}.
$$
Hence, $|\al^j_{t_1}-\al^j_{t_2}|\geq L-\frac{2L}{2+5T}=\frac{L}{1+0.4/T}$
for all distinct $t_1,t_2\in[k]$.

\noindent
(ii) For every $j\in[k'-1]$, for every $t\in[k]$, there is a distinct $t'$ such that
$$
|w_t-\hw^j_{t'}|=O\Bigl(\frac{\dt\w^{1.5}k^5}{w_{\min}^2\zeta^2}\Bigr), \qquad
|z_j^\dagger p^t-\hal^j_{t'}|=O\Bigl(\frac{\sqrt{\dt}}{w_{\min}^{1.5}\zeta\sqrt{n}}\Bigr)
\ \ \text{and is at most}\ \frac{L}{2+5T}.
$$
\end{lemma}

\begin{proof}
Assume that the event stated in Lemma~\ref{jlprojn} happens. Then the inputs to
$\learn$ in steps A2 and A3 are ``valid'', i.e., satisfy the assumptions stated in
Algorithm $\learn$.
Plug in $\procacc=\dt$ and $\procerrp=\frac{1}{6\w k}$ in Lemma~\ref{helper}. Taking the
union bound over all the $b_j$s and the $z_j$s, we obtain that the probability that
$\learn$ fails on some input, when all the $b_j$s and $z_j$s are valid is at most
$\frac{1}{3\w}$. The lemma follows from
Lemma~\ref{helper} by noting that
$\frac{2048kH\dt}{w_{\min}}=O\Bigl(\frac{\sqrt{\dt}}{w_{min}^{1.5}\zeta\sqrt{n}}\Bigr)$
and is at most $\frac{L}{24T}$, and
$L/24T+L/8T\leq L/(2+5T)$.
\end{proof}

\begin{lemma} \label{reconcile}
With probability at least $1-\frac{1}{\w}$, for every $j=1,\ldots,k'-1$
$\vro^j$ is a well-defined function and $\vro^j(\sg^{k'}(t))=\sg^j(t)$ for every
$t\in[k]$.
\end{lemma}

\begin{proof}{Lemma~\ref{reconcile}}
Assume that the events in Lemmas~\ref{jlprojn} and~\ref{blearn} occur.
Fix $j\in[k'-1]$. We call a point $\al^j_{t_1}b_j+\al^{k'}_{t_2}b_{k'}$ a grid-$j$ point.
Call this grid point ``genuine'' if there exists $t\in[k]$ such that
$\sg^j(t)=t_1$ and $\sg^{k'}(t)=t_2$, and ``fake'' otherwise.
The distance between any two grid-$j$ points is at least $L/(1+0.4/T)$ (by
Lemma~\ref{blearn}). So the probability there is a pair of genuine and fake grid-$j$
points whose projections on $z_j$ are less than $L/(T+0.4)$ away is at most
$k^3\cdot\frac{2}{\pi}\arcsin\bigl(\tfrac{1}{T}\bigr)
\leq k^3\cdot\frac{2}{\pi}\cdot\frac{6}{5T}\leq\frac{1}{3\w k}$.
Therefore, with probability at least $1-\w$, the events in Lemma~\ref{jlprojn} and
Lemma~\ref{blearn} happen, and for all $j\in[k'-1]$, every pair of genuine and fake
grid-$j$ points project to points on $z_j$ that are at least $L/(T+0.4)$ apart. We
condition on this in the sequel.

Now fix $j\in[k'-1]$ and consider any pair $t_1,t_2\in[k]^2$.
Let $g$ be the grid-$j$ point $b_j\al^j_{t_1}+b_{k'}\al^{k'}_{t_2}$
We show that $\vro^j(t_2)=t_1$ iff $g$ is a genuine grid-$j$ point.
If $g$ is genuine, let $t$ be such that $\sg^j(t)=t_1,\ \sg^{k'}(t)=t_2$.
Let $p'$ be the projection of $p^t$ on $\Span(b_j,b_{k'})$. By Lemma~\ref{blearn}, we have
that $\|p'-g\|_2\leq\frac{\sqrt{2}L}{2+5T}$. Also, there exists $t'\in[k]$ such that
$|\hal^j_{t'}-z_j^\dagger p^t|\leq\frac{L}{2+5T}$. Since $z_j^\dagger p'=z_j^\dagger p^t$,
this implies that
$|z_j^\dagger g-\hal^j_{t'}|\leq|\hal^j_{t'}-z_j^\dagger p^t|+|z_j^\dagger(p'-g)|
\leq\frac{(\sqrt{2}+1)L}{2+5T}$ and so $\vro^j(t_2)=t_1$.

Now suppose $g$ is fake but $|z_j^\dagger g-\hal^j_{t'}|\leq(\sqrt{2}+1)L/(2+5T)$ for some
$t'\in [k]$. Let $t\in[k]$ be
such that 
$|\hal^j_{t'}-z_j^\dagger p^t|\leq\frac{L}{2+5T}$. Let $g'$ be the genuine grid point
$b_j\al^j_{\sg^j(t)}+b_{k'}\al^{k'}_{\sg^{k'}(t)}$.
So $|z_j^\dagger g'-\hal^j_{t'}|\leq(\sqrt{2}+1)L/(2+5T)$, and hence
$|z_j^\dagger(g-g')|\leq\frac{2(\sqrt{2}+1)L}{2+5T}<\frac{L}{0.4+T}$ which is a
contradiction.
\end{proof}

\begin{proofof}{Theorem~\ref{mainthm}}
We condition on the fact that all the ``good'' events stated in Lemmas~\ref{rest},
\ref{Aest}, \ref{jlprojn}, \ref{blearn}, and \ref{reconcile} happen. The probability of
success is thus $1-O\bigl(\frac{1}{\w}\bigr)$. The sample-size bounds follow from the
description of the algorithm.
For notational simplicity, let $\sg^{k'}$ be the identity permutation, i.e.,
$\sg^{k'}(t)=t$ for all $t\in[k]$. So by Lemma~\ref{reconcile}, we have
$\vro^j(t)=\sg^j(t)$ for every $j\in[k'-1]$ and $t\in k$.

For $t=1,2,\dots,k$, define
$\bar{p}^t=\tilde{r}+\sum_{j=1}^{k'}b_j^\dagger(p^t-\tr)b_j=\tr+\Pi_{\tA}(p^t-\tr)$.
Fix $t\in[k]$. Then
\begin{align*}
\|p^t - \tilde{p}^t\|_1 & \le \|p^t-\hp^t\|_1+\|\hp^t-\tp^t\|_1 
\leq  2\|p^t - \hat{p}^t\|_1  
\le 2\left(\|p^t - \bar{p}^t\|_1 + \|\bar{p}^t - \hat{p}^t\|_1\right). \\
\text{We have}\ \ \|p^t - \bar{p}^t\|_2 & = \Bigl\|r - \tilde{r} + (p^t - r) -
    \sum_{j=1}^{k'} b_j^\dagger(p^t-\tilde{r}) b_j\Bigr\|_2 \\
& =  \left\|r - \tilde{r} + \Pi_{\tilde{A}} (\tilde{r} - r) +
          (p^t - r) - \Pi_{\tilde{A}}(p^t - r)\right\|_2 \\
& \le  2\cdot\|r - \tilde{r}\|_2 +
   \cdot \|\Pi_A - \Pi_{\tilde{A}}\|_{\op}\cdot\|p^t - r\|_2 
\le  \frac{\delta}{12\sqrt{n}} + \frac{8\sqrt{2\delta}}{w_{\min}\zeta\sqrt{n}}. \\
\text{Also}\ \ 
\|\bar{p}^t - \hat{p}^t\|_2 
& \le \Bigl\|\sum_{j=1}^{k'}\bigl(b_j^\dagger p^t-\al^j_{\sigma^j(t)}\bigr) b_j\Bigr\|_2 
= O\Bigl(\frac{\sqrt{k\dt}}{w_{\min}^{1.5}\zeta\sqrt{n}}\Bigr)
\end{align*}
where the last equality follows from Lemma~\ref{blearn}.
Thus, $\|p^t-\tp^t\|_1=O\bigl(\frac{\sqrt{k\dt}}{w_{\min}^{1.5}\zeta}\bigr)$.
Also, we have $|w_t-\tw_t|=O\bigl(\frac{\dt\w^{1.5}k^5}{w_{\min}^2\zeta^2}\bigr)$ by
Lemma~\ref{blearn}.
Finally, note that 
\begin{align*}
\Tran(w,P; \tw,\tP) & \leq
\frac{1}{2}\Bigl(\sum_{t=1}^k \min\{w_t,\tw_t\}\max_t\|p^t-\tp^t\|_1 
+\|w-\tw\|_1\max_{t\neq t'}\|p^t-\tp^{t'}\|_1\Bigr) \\
& \leq \frac{1}{2} \Bigl(\sum_{t=1}^k \min\{w_t,\tw_t\}\max_t\|p^t-\tp^t\|_1 
+\|w-\tw\|_1\bigl(\max_{t\neq t'}\|p^t-p^{t'}\|_1 +\max_t\|p^t-\tp^t\|_1\bigr)\Bigr) \\
& \leq \max_t\|p^t-\tp^t\|_1 +\|w-\tw\|_1\cdot\frac{2}{w_{\min}}
= O\Bigl(\frac{\sqrt{k\dt}}{w_{\min}^{1.5}\zeta}\Bigr). 
\end{align*}
The running time is dominated by the time required to compute the spectral decomposition
in step A1.3, the calls to \learn in steps A2.2 and A3.2, and the time to compute
$\tp^t$ in step A3.4. The other steps are clearly polytime.
As noted in Remark~\ref{algremk}, it suffices to compute a decomposition such that  
$\|\tM-\tR-\sum_{i=1}^n\ld_iv_iv_i^\dagger\|=O\bigl(\frac{\dt}{n}\bigr)$; this takes time 
$\poly\bigl(n,\ln(n/\dt)\bigr)$. The LP used in step A3.4 is of polynomial size, and hence
can be solved in polytime. Procedure \learn requires solving \eqref{ajlp}; again, an
approximate solution suffices and can be computed in polytime. Theorem~\ref{onedthm} 
proves that the one-dimensional learning problem can be solved in polytime; hence, \learn
takes polytime.  
\end{proofof}

\section{The one-dimensional problem: learning mixture sources on \boldmath [0,1]} 
\label{onedim} \label{ONEDIM}
In this section, we supply the key subroutine called upon in step L2 of Algorithm
\learn, which will complete the description of Algorithm~\ref{mainalg}.
We are given a $k$-mixture source $\bigl(w,\pi_x(P)\bigr)$ on
$\bigl[-\frac{1}{2},\frac{1}{2}\bigr]$. (Recall that \learn invokes the procedure for the   
mixture $\bigl(w,\pi_{a/2H}(P)\bigr)$ where $\|a\|_\infty\leq H$.) 
It is clear that we \textit{cannot} in general reconstruct this mixture source with an
aperture size that is independent of $n$, let alone aperture $2k-1$. 
However, our goal is somewhat different and more modest. We seek to reconstruct the
$k$-spike distribution $\bigl(w,\E[\pi_x(P)]\bigr)$, 
and we show that this {\it can} be
achieved with aperture $2k-1$ (which is the smallest aperture at which this is
information-theoretically possible). 

It is easy to obtain a $(2k-1)$-snapshot from $(w,\pi_x(P))$ given a
$(2k-1)$-snapshot from $(w,P)$ by simply replacing each item $i\in[n]$ that appears in the 
snapshot by $x_i$. We will assume in the sequel that every constituent
$\pi_x(p^t)$ is supported on $[0,1]$, 
which is simply a translation by $\frac{1}{2}$.  

To simplify notation, we use $\bdth=\bigl(\vth,(q^1,\ldots,q^k)\bigr)$ to denote the
$k$-mixture source on $[0,1]$, and $\bigl(\vth,\al=(\al_1,\ldots,\al_k)\bigr)$ to denote
the corresponding $k$-spike distribution, 
where $\al_i\in[0,1]$ is the expectation of $q^i$ for all $i\in[k]$. 
We equivalently view $(\vth,\al)$ as a $k$-mixture source
$\bigl(\vth,(\cn^1,\ldots,\cn^k)\bigr)$ on $\{0,1\}$: each $\cn^i$ is a
``coin'' whose bias is $\cn^i_1=\al_i$. In Section~\ref{k-spikes-1D}, we describe how to 
learn such a {\em binary} mixture source from its $(2k-1)$-snapshots (see
Algorithm~\ref{kspike-alg} and Theorem~\ref{kspike-thm}).  
Thus, if we can obtain $(2k-1)$-snapshots from the binary source
$\bigl(\vth,(\cn^1,\ldots,\cn^k)\bigr)$ (although our input is $\bdth$) 
then Theorem~\ref{kspike-thm} would yield the the desired result. 
We show that this is indeed possible,
and hence, obtain the following result (whose proof appears at the end of
Section~\ref{k-spikes-1D}).  

\begin{theorem}\label{k-dists-from-ap-1D} \label{onedthm}
Let $\bdth=\bigl(\vth,(q^1,\ldots,q^k)\bigr)$ be a $k$-mixture source on $[0,1]$, and 
$(\vth,\al)$ be the corresponding $k$-spike distribution.
Let $\onedzeta=\min_{j\neq j'}|\al_j-\al_{j'}|$.
For any $s<\onedzeta$ and $\onederrp>0$, using $3k 2^{4k} s^{-4k} \ln (4k/\onederrp)$
$(2k-1)$-snapshots from source $\bdth$, one can compute in polytime a $k$-spike
distribution $(\tvth,\tal)$ on $[0,1]$ such that 
$\Tran(\vth,\al;\tvth,\tal)\leq 1024ks^{1/(4k)}$ with probability at least $1-\onederrp$.  
\end{theorem}

\subsection{Learning a binary \boldmath {\large $k$}-mixture source} \label{k-spikes-1D}
Recall that $\bigl(\vth,(\cn^1,\ldots,\cn^k)\bigr)$ denotes the binary $k$-mixture source,
and $\al_i=\cn^i_1$ is the bias of the $i$-th ``coin''.
We can collect from each $(2k-1)$-snapshot a random variable 
$0 \leq X \leq 2k-1$ denoting the number of times the outcome ``$1$'' occurs in the
snapshot. 
Thus,
$\Pr[X=i]=\binom{2k-1}{i}\sum_{j=1}^k \vth_j
\al_j^{i}(1-\al_j)^{2k-1-i}$.
Our objective is to use these
statistics to reconstruct, in transportation distance (see
Section~\ref{trans-defn}), the binary source (i.e., the mixture weights and the $k$
biases). 
Now consider the equivalent $k$-spike distribution $(\vth,\al)$.
The $i$-th moment, and (what we call) the $i$-th {\em normalized binomial moment} (NBM) of
this distribution are 
\be
g_i(\vth,\al) = \sum_{j=1}^k \vth_j \alpha_j^i, 
\qquad 
\nu_i(\vth,\al)=\sum_{j=1}^k \vth_j\al_j^{i}(1-\al_j)^{2k-1-i} 
\label{g-and-nu}
\ee
Up to the factors $\binom{2k-1}{i}$ the NBMs are precisely the
statistics of the random variable $X$ 
and so our objective in this section can be restated as: use the
empirical NBMs to reconstruct the $k$-spike distribution
$(\vth,\al)$.

Let $g(\vth,\al)=\bigl(g_i(\vth,\al)\bigr)_{i=0}^{2k-1}$
and $\nu(\vth,\al)=\bigl(\nu_i(\vth,\al)\bigr)_{i=0}^{2k-1}$ 
denote the row-vectors of the first $2k-1$ moments and NBMs respectively of $(\vth,\al)$. 
For an integer $b>0$ and a vector $\beta=(\beta_1,\ldots,\beta_\ell)$, let
$A_b(\beta)$ be the $\ell \times b$ matrix
$(A_b(\beta))_{ij}=(1-\beta_i)^{b-1-j}\beta_i^j$ (with $1
\leq i \leq \ell$ and $0 \leq j \leq b-1$). 
Analogously, let
$V_b(\beta)$ be the $\ell \times b$ ``Vandermonde'' matrix
$(V_b(\beta))_{ij}=\beta_i^j$ (with $1 \leq i \leq \ell$ and $0
\leq j \leq b-1$). 
Let $\Pas$ be the $2k \times 2k$ lower-triangular ``Pascal
triangle'' matrix with non-zero entries $\Pas_{ij}={2k-1-j \choose i-j}$
for $0 \leq j \leq 2k-1$ and $j \leq i \leq 2k-1$. Then
$V_{2k}(\alpha)=A_{2k}(\alpha) \Pas$,
$\nu(\vth,\al)=\vth A_{2k}(\al)$, and 
$g(\vth,\al)=\vth V_{2k}(\al)=\nu(\vth,\al)\Pas$.

In our algorithm it is convenient to use the empirical
ordinary moments, but what we obtain are actually the
empirical NBMs, so we need the following lemma.

\begin{lemma} \label{lem-pas}
$\| \Pas \|_\op \leq 4^k/\sqrt{3}$. 
\end{lemma}

\begin{proof}
The non-zero entries in column $j$ of $\Pas$ are ${m \choose \ell}$ for
$\ell=0,\ldots,m=2k-1-j$. 
Therefore, $\| \Pas \|_\op \leq \| \Pas \|_F
=\sqrt{\sum_{m=0}^{2k-1}\sum_{\ell=0}^m{\binom{m}{\ell}}^2}
=\sqrt{\sum_{m=0}^{2k-1}\binom{2m}{m}}\leq\sqrt{\sum_{m=0}^{2k-1} 2^{2m}}$.
\end{proof}

Our algorithm uses two input parameters $\onedzeta$ and $\xi$ as input, and the empirical
NBM vector $\tnu$, which we convert to an empirical moment vector $\tg$ by multiplying by $\Pas$.
Since we infer (in the sampling limit) the locations of the $k$
spikes exactly, there is a singularity in the process when
spikes coincide. 
So we assume a minimum separation between spikes: $\onedzeta =
\min_{j\ne j'} |\al_j - \al_{j'}|$. (It is of course
possible to simply run a doubling search for sufficiently
small $\onedzeta$, but the required accuracy in the moments, and
hence sample size, does increase as $\onedzeta$ decreases.)
We also assume a bound $\xi$ on the accuracy of our empirical statistics. 
(When we 
utilize Theorem~\ref{kspike-thm} to obtain Theorem~\ref{k-dists-from-ap-1D}, $\xi$ is a
consequence, and not an input parameter). 
We require that
\be
\|\tilde{\nu}-\nu(\vth,\al)\|_2\leq\xi 4^{-k}\sqrt{3},
\qquad \qquad \xi \leq \onedzeta^{2k} \label{xi}
\ee

\begin{theorem} \label{kspike-thm}
There is a polytime algorithm that receives as input $\onedzeta, \xi$, an empirical
NBM vector $\tilde{\nu} \in \RR^{2k}$ satisfying \eqref{xi}, and outputs a $k$-spike
distribution $(\tvth,\tal)$ on $[0,1]$ 
such that $\Tran(\vth,\al;\tvth,\tal)\leq O(\xi^{\Omega(1/k^2)})$.
\end{theorem}

We first show the information-theoretic feasibility of
Theorem~\ref{kspike-thm}: the transportation distance
between two probability measures on
$[0,1]$ is upper bounded by (a moderately-growing function of) 
the Euclidean distance between their moment maps. 
(To use Lemma~\ref{infrecon} to prove Theorem~\ref{kspike-thm}, we 
have to also show how to compute $\tvth$ and $\tal$ from $\tg$ such that
$\|\tg-g(\tvth,\tal)\|_2$, and hence, $\|g(\vth,\al)-g(\tvth,\tal)\|_2$ is small.)  

\begin{lemma}\label{new-reconinf-in-new-place} \label{infrecon}
For any two (at most) $k$-spike distributions $(\vth,\al)$ $(\tvth,\tal)$ on 
$[0,1]$,
\[\|g(\vth,\al)-g(\tvth,\tal)\|_2 \geq \frac{1}{(2k-1)^{4k} 2^{8k-5}}\cdot
\left(\Tran(\vth,\al;\tvth,\tal)\right)^{4k-2}.\]
\end{lemma}

Lemma~\ref{infrecon} can be geometrically interpreted as follows.
The point $g(\vth,\al)$ is in the convex hull of the moment curve and is therefore, by
Caratheodory's theorem, expressible as a convex combination
of $2k$ points on the curve. However, this point is special
in that it belongs to the collection of points expressible as a convex combination of 
merely $k$ points of the curve. Lemma~\ref{infrecon} shows that $g(\vth,\al)$ is in fact
\textit{uniquely} expressible in this way, and that moreover this combination is stable:
any nearby point in this collection can only be expressed as a very similar convex
combination. We utilize the following lemma, which can be understood as a global curvature
property of the moment curve; we defer its proof to Section~\ref{interpolation-proof}.
We prove a partial converse of Lemma~\ref{infrecon} in Section~\ref{lbound}, and hence
obtain a sample-size lower bound that is exponential in $k$. The moment curve plays a
central role in convex and polyhedral geometry~\cite{Barvinok02}, but as far as we know
Lemmas~\ref{infrecon} and~\ref{lm:interpolation-in-new-place} are new, and may be of
independent interest.  

\begin{lemma}\label{lm:interpolation-in-new-place}
Let $0\leq \beta_1<\ldots<\beta_{\kappa+1}\leq 1$, $\ell\in[\kp]$, and
$s=\beta_{\ell+1}-\beta_\ell$. 
Let $\gm(x)= \sum_{i=0}^{\kappa} \gamma_i x^i$
be a real polynomial of degree $\kappa$ evaluating to $1$ at the points
$\beta_1,\ldots,\beta_\ell$ and evaluating to $0$ at the points
$\beta_{\ell+1},\ldots,\beta_{\kappa+1}$. Then 
$\sum_{i=0}^\kappa \gamma_i^2\leq \kappa^2 2^{4\kappa-1} s^{-2\kappa}$.
\end{lemma} 

\begin{proofof}{Lemma~\ref{infrecon}}
Denote $\{\alpha_1 , \ldots ,\alpha_k \} \cup  \{\tal_1 , \ldots ,\tal_k \}$ 
by $\oba=\{\oba_1 , \ldots , \oba_{K}\}$ where
$\oba_1 < \ldots < \oba_{K}$. 
Define $\obt_i=\sum_{j:\al_j=\oba_i}\vth_j-\sum_{j:\tal_j=\oba_i}\tvth_j$ for $i\in[K]$. 
Let $\obt\in \R^{K}$ be the row vector $(\obt_1,\ldots,\obt_{K})$. 
Let $\eta=\Tran(\vth,\al;\tvth,\tal)$.
So we need to show 
that $\|\obt V_{2k}(\oba)\|_2 \geq \frac{1}{(2k-1)^{4k} 2^{8k-5}}\cdot\eta^{4k-2}$.
It suffices to show that 
$\|\obt V_{K}(\oba)\|_2\geq\frac{1}{(K-1)^{2K} 2^{4K-5}}\cdot\eta^{2K-2}$.
There is an $1 \leq \ell < K$ such that
$ \bigl|\sum_{i=1}^\ell \obt_i \bigr| \cdot (\oba_{\ell+1}-\oba_\ell) \geq
\eta/(K-1).$ Let $\delta=\sum_{i=1}^\ell \obt_i$; without loss of generality
$\delta \geq 0$, and note that $\delta \leq 1$. Let
$s=\oba_{\ell+1}-\oba_\ell$, so $(K-1) \delta s \geq \eta$.
Denote row $i$ of a matrix $Z$ by $Z_{i*}$ and column $j$ by $Z_{*j}$.
We lower bound $\|\obt V_{K}(\oba)\|_2$, by considering its minimum value 
under the constraints $\sum_{i=1}^{\ell}\obt_i=\dt$ and $\sum_{i=1}^{K}\obt_i=0$.

A vector $y^\dagger=\obt V_{K}(\oba)$ minimizing $\|y\|_2$ must
be orthogonal to $V_{K}(\oba)_{i*}-V_{K}(\oba)_{i'*}$ if $1 \leq i < i' \leq
\ell$ or if $\ell+1 \leq i < i' \leq K$. This means that there
are scalars $c$ and $d$ such that 
$V_{K}(\oba)y=c(\sum_{j=1}^\ell e_j)+d(\sum_{j=\ell+1}^{K} e_j)$, where vector
$e_j\in\R^{K}$ has a 1 in the $j$-th position and 0 everywhere else.
Therefore, $y = c\gamma + d\gamma'$,
where $\gamma = \sum_{j=1}^{\ell} (V_{K}(\oba)^{-1})_{*j}$ and
$\gamma' =  \sum_{j=\ell+1}^{K}(V_{K}(\oba)^{-1})_{*j}$.
At the same time
\vspace{-1ex}
$$
\delta = \sum_{i=1}^\ell \obt_i =
\obt V_{K}(\oba) \gamma =  y^\dagger\gamma = c\|\gm\|_2^2+d\gm'^\dagger \gm \qquad
-\delta = \sum_{i=\ell+1}^{K} \obt_i = 
\obt V_{K}(\oba) \gamma' = y^\dagger\gamma' = c\gm^\dagger \gm'+d\|\gm'\|_2^2
$$
and hence, $\|y\|_2^2 = y\cdot(c\gamma + d \gamma') = (c-d)\delta$.
Solving for $c, d$ gives
$
c - d =  
\frac{\delta \|\gamma+\gamma'\|_2^2}
{\|\gamma\|_2^2 \cdot \|\gamma'\|_2^2 -
(\gamma^\da \cdot\gamma')^2}$.

First we examine the numerator of $c-d$. 
Like any combination of the columns of $V_{K}(\oba)^{-1}$, $\gm+\gm'$ is the list of  
coefficients of a polynomial of degree $K-1$, in the basis
$1, x, \ldots, x^{K-1}$.
By definition, $\gamma+\gamma' = \sum_j (V_{K}(\oba)^{-1})_{*j}$, which
is to say that for every $i$, $V_{K}(\oba)_{i*}\cdot(\gamma+\gamma')=1$. So the
polynomial $\gamma+\gamma'$ evaluates to $1$ at every $\oba_i$. It
can therefore only be the constant polynomial $1$;
this means that $(\gamma+\gamma')_i=1$ if $i=0$, and
$(\gamma+\gamma')_i=0$ otherwise. Thus
$\|\gamma+\gamma'\|_2^2=1$.

Next we examine the denominator, which we upper bound by
$\|\gamma\|_2^2 \cdot \|\gamma'\|_2^2$. 
When interpreted as a polynomial, $\gamma$ takes the
value $1$ on a nonempty set of points
$\oba_1,\ldots,\oba_\ell$ separated by the positive distance
$s=\oba_{\ell+1}-\oba_\ell$ from another nonempty set of
points $\oba_{\ell+1},\ldots,\oba_{K}$ upon which it takes
the value $0$. 
Observe that if the polynomial was required
to change value by a large amount within a short interval,
it would have to have large coefficients. A converse to this 
is the inequality stated in Lemma~\ref{lm:interpolation-in-new-place}.
Using this to bound $\|\gm\|_2^2$ and $\|\gm'\|_2^2$, 
and since $\dt s\geq \eta/(K-1)$,
we obtain that 
$$
\|y\|_2^2 = (c-d)\delta 
\geq\frac{\delta^2}{\|\gamma\|_2^2 \cdot \|\gamma'\|_2^2}  
\geq \frac{\delta^2}{((K-1)^2 2^{4K-5} s^{-2K+2})^2} 
\geq \frac{\eta^{4K-4}}{(K-1)^{4K} 2^{8K-10}}. 
\qedhere
$$
\end{proofof}

We now define the algorithm promised by Theorem~\ref{kspike-thm}.
To give some intuition, 
suppose first that we are given the true moment vector
$g(\vth,\al)=\vth V_{2k}(\al)$.
Observe that there is a common vector
$\ld=(\lambda_0,\ldots,\lambda_k)^\da$ of length $k+1$
that is a dependency among
every $k+1$ adjacent columns of $V_{2k}(\alpha)$.
In other words, letting $\Ld=\Ld(\ld)$ denote the $2k \times k$ matrix with
$\Lambda_{ij}=\lambda_{i-j}$ (for $0 \leq i<2k,\ 0\leq j<k$ and with the
understanding $\lambda_\ell=0$
for $\ell \notin \{0,\ldots,k\}$), $V_{2k}(\alpha)\Lambda = 0$.
Thus $g(\vth,\al)\Ld=\vth V_{2k}(\al)\Ld=0$. Overtly this is a
system of $2k$ equations to determine $\lambda$. But we
eliminate the redundancy in $\Lambda$ by forming
the $k \times (k+1)$ matrix $G = G(g(\vth,\al))$ defined by
$G_{ij}=g(\vth,\al)_{i+j}$ for
$i=0,\ldots,k-1$ and $j=0,\ldots,k$, and then solve the system of linear
equations $G \lambda = 0$ to obtain
$\ld$. This system does not have a unique solution, so in the sequel
$\ld$ will denote a solution with
$\ld_k=1$. For each $i=1,\ldots,k$, we have
$\bigl(V_{2k}(\al)\Ld\bigr)_{i,1}=\sum_{\ell=0}^k\ld_\ell{\al_i}^\ell=0$.
This implies that we can obtain the $\al_i$ values by computing
the roots of the polynomial $P_\ld(x):=\sum_{\ell=0}^k \lambda_\ell x^\ell$.
Once we have the $\al_i$'s, we can compute $\vth$ by
solving the linear system $yV_{2k}(\al)=g(\vth,\al)$ for $y$.

Of course, we are actually given $\tg$ rather than the true
vector $g(\vth,\al)$. So we need to control the error
in estimating first $\al$ and then $\vth$.
The learning algorithm is as follows.

{\small \vspace{10pt} 
\begin{algorithm} \label{kspike-alg}
\vspace{-5pt} \hrule \vspace{5pt}
Input: parameters $\xi, \onedzeta$ and empirical moments $\tg$ such that
$\|\tg-g(\vth,\al)\|_2\leq\xi$. 

\noindent
Output: a $k$-spike distribution $(\tvth,\tal)$
\end{algorithm}
\begin{labellist}[B]
\item Solve the minimization problem:
\vspace{-5pt}
\begin{equation}
\text{minimize} \quad \|x\|_1 \quad \text{s.t.} \quad
\|G(\tg)x\|_1\leq 2^k k \xi, \quad x_k=1 \tag{P} \label{ld-lp}
\end{equation}
\vspace{-15pt}

which can be encoded as a linear program and hence solved in polytime,
to obtain a solution $\tld$. 
Observe that since $G(\tg)$ has $k+1$
columns and $k$ rows, there is always a feasible solution.

\item Let $\bal_1,\ldots,\bal_k$ be the (possibly complex)
roots of the polynomial
$P_\tld$. Thus, we have $V_{2k}(\bal)\Ld(\tld)=0$.
We map the roots to values in $[0,1]$ as follows.
Let $\e=\frac{4}{\onedzeta}(2k\xi)^{1/k}$.
First we compute $\hal_1,\ldots,\hal_k$ such
that $|\hal_i-\bal_i|\leq\e$ for all
$i$, in time $\poly\bigl(\log(\frac{1}{\e})\bigr)$, using
Pan's algorithm~\cite[Theorem 1.1]{Pan96}\footnote{The
theorem requires that the complex roots lie
within the unit circle and that the coefficient of the
highest-degree term is 1;
but the discussion following it in~\cite{Pan96} shows that this
is essentially without loss of generality.}.
We now set $\tal_i=\max\{0,\min\{\real(\hal_i),1\}\}$. 

\item Finally, we set $\tvth$ to be the
  row-vector $y$
that minimizes $\|yV_{2k}(\tal)-\tg\|_2^2$ subject to
$\|y\|_1=1$, $y \geq 0$.
Note that this is a convex quadratic program that can be solved exactly in
polytime~\cite{ChungM81}. 
\end{labellist}
\vspace{-2pt}
\hrule
}

\vspace{8pt}
We now analyze Algorithm~\ref{kspike-alg} and justify Theorem~\ref{kspike-thm}.
Recall that $\onedzeta=\min_{j\neq j'}|\al_j-\al_{j'}|$.
We need the following lemma, whose proof appears in Section~\ref{sens-proof}.

\begin{lemma} \label{sensitivity-in-new-place} \label{sens}
The weights $\tvth$ satisfy
$\|\tvth V_{2k}(\tal)-\tg\|_2\leq
\|g(\vth,\al)-\tg\|_2+\frac{(8k)^{5/2}}{\onedzeta}\cdot(2k\xi)^{1/k}$.
\end{lemma}

\begin{proofof}{Theorem~\ref{kspike-thm}}
We call Algorithm~\ref{kspike-alg} with $\tg=\tnu\Pas$. 
By Lemma~\ref{lem-pas}, we obtain that $\|\tg-g(\vth,\al)\|_2 \leq \xi$,
and by Lemma~\ref{sensitivity-in-new-place}, we have that 
$\|g(\vth,\al)-\tvth V_{2k}(\tal)\|_2\leq
2\|g(\vth,\al)-\tg\|_2+\frac{8}{\onedzeta}\cdot(8k)^{3/2}\bigl(2k\xi\bigr)^{1/k}$.
Coupled with Lemma~\ref{new-reconinf-in-new-place} and since 
$\xi\leq\onedzeta^{2k}$, we obtain that
\begin{align*}
\Tran(\vth,\al;\tvth,\tal)
& \leq \left[ (2k-1)^{4k} 2^{8k-5}\|g(\vth,\al)-g(\tvth,\tal)\|_2 \right]^\frac{1}{4k-2} \\
& \leq  \left[ (2k-1)^{4k} 2^{8k-5} 
  \Bigl(2\xi+\frac{(8k)^{\frac{5}{2}}}{\onedzeta}\bigl(2k\xi\bigr)^{\frac{1}{k}}\Bigr)
  \right]^{\frac{1}{4k-2}} \\
& \leq \left[ (2k-1)^{4k} 2^{8k-5}
  \Bigl(2\xi+(8k)^{5/2}\bigl(2k\sqrt{\xi}\bigr)^{1/2k}\Bigr)
  \right]^{\frac{1}{4k-2}} 
\leq 1024\cdot k\xi^{\frac{1}{8k^2}}. \hspace*{0.5in} \qedhere
\end{align*}
\end{proofof}

\begin{proofof}{Theorem~\ref{onedthm}}
We convert $\bdth$ to the corresponding binary source
$\bigl(\vth,(\cn^1,\ldots,\cn^k)\bigr)$ by randomized rounding.
Given a $(2k-1)$-snapshot $z=(z_1,\ldots,z_{2k-1})\in[0,1]^{2k-1}$ from
$\bdth$, we obtain a $(2k-1)$-snapshot from the binary source as follows.
We choose $2k-1$ independent values $a_1,\ldots,a_{2k-1}$ uniformly at random
from $[0,1]$ and set $X_i=1$ if $z_i\geq a_i$ and 0 otherwise for all $i\in[2k-1]$. 
Note that if $q^j$ is the constituent generating the $(2k-1)$-snapshot $z$, then 
$\Pr[X_i=1|q^j]=\E[X_i|q^j]=\al_j$, and so
$X_1,\ldots,X_{2k-1}$ is a random $(2k-1)$-snapshot from the above binary
source. 

Now we apply Theorem~\ref{kspike-thm}, setting $\xi=s^{2k}$.
Let $\tnu$ be the empirical NBM-vector obtained from the $(2k-1)$-snapshots of the above
binary source (i.e., $\tnu_i=\binom{2k-1}{i}^{-1}\cdot$
(frequency with which the $(2k-1)$-snapshot has exactly $i$ 1s)).  
The stated sample size ensures, via a Chernoff bound, that 
{$\Pr\bigl[|\tilde{\nu}_i - \nu(\vth,\al)_i|\geq\frac{\xi 4^{-k}}{\sqrt{6k}}\bigr]
<\frac{\onederrp}{2k}$} 
for all $i=0,\ldots,2k-1$.
Hence, with probability at least $1-\onederrp$, we have 
$\|\tnu-\nu(\vth,\al)\|_2\leq\sqrt{2k}\cdot\|\tilde{\nu}-\nu(\vth,\al)\|_\infty\leq
\xi 4^{-k}/\sqrt{3}$.
\end{proofof}

\subsection{Proofs of Lemma~\ref{lm:interpolation-in-new-place} and Lemma~\ref{sens}} 
\label{interpolation-proof} \label{sens-proof}

\begin{proofof}{Lemma~\ref{lm:interpolation-in-new-place}}
There are two easy cases to dismiss before we reach
the more subtle part of this lemma. The first easy
case is $\ell=1$. In this case $\gamma$ is a single
Lagrange interpolant:
$\gamma(x)=\prod_{j=2}^{\kappa+1} \frac{x-\beta_j}{\beta_1-\beta_j}$.
For $0 \leq i \leq \kappa$ let
$e^{\kappa}_i(\beta_2,\ldots,\beta_{\kappa+1})$ be
the $i$'th \textit{elementary symmetric mean},
\[ e^{\kappa}_i(\beta_2,\ldots,\beta_{\kappa+1}) 
= \frac{1}{{\kappa \choose i}}\sum_{S\sse\{2,\ldots,\kappa+1\}: |S|=i} \prod_{j \in S} \beta_j \]
and observe that for all $i$, $0 \leq
e^{\kappa}_i(\beta_2,\ldots,\beta_{\kappa+1}) \leq 1$. Now
\[ \gamma(x)=
\Bigl( \prod_{j=2}^{\kappa+1} \frac{1}{\beta_1-\beta_j} \Bigr)
\sum_{i=0}^{\kappa} (-1)^{\kappa-i}{\kappa \choose i}
e^{\kappa}_{\kappa-i}(\beta_2,\ldots,\beta_{\kappa+1})x^{i}
\]
So $\sum_{i=0}^\kp \gamma_i^2 
=\left( \prod_{j=2}^{\kappa+1} \frac{1}{\beta_1-\beta_j} \right)^2
\sum_{i=0}^\kp \left( {{\kappa \choose i}}
e^{\kappa}_i(\beta_2,\ldots,\beta_{\kappa+1})\right)^2 
\leq s^{-2\kappa} \sum_{i=0}^\kp {{\kappa \choose i}}^2 
= {2\kappa \choose \kappa} s^{-2\kappa}$.

The second easy case is $\ell=\kappa$; this is almost
as simple. Merely note that the above argument
applies to the polynomial $1-\gamma$, so that we have
only to allow for the possible increase of
$|\gamma_0|$ by $1$. Hence $\sum_{i=0}^\kp \gamma_i^2 \leq 4
{2\kappa \choose \kappa} s^{-2\kappa}$.

We now consider the less trivial case of
$1<\ell<\kappa$. The difficulty here is that the
Lagrange interpolants of $\gamma$ may have very large
coefficients, particularly if among $\beta_1,\ldots,
\beta_\ell$ or among $\beta_{\ell+1},\ldots,
\beta_{\kappa+1}$ there are closely spaced roots, as
well there may be. We must show that these large
coefficients cancel out in $\gamma$.

The trick is to examine not $\gamma$ but $\partial
\gamma /\partial x$. The roots of the derivative
interlace the two sets on which $\gamma$ is constant,
which is to say, with $\beta'_1 \leq \ldots \leq
\beta'_{\kappa-1}$ denoting the roots of $\partial
\gamma /\partial x$, that for $j<\ell$, $\beta_j \leq
\beta'_j \leq \beta_{j+1}$, and for $j \geq \ell$,
$\beta_{j+1} \leq \beta'_j \leq \beta_{j+2}$. In
particular, none of the roots fall in the interval
$(\beta_\ell,\beta_{\ell+1})$. For some constant $C$ we
can write $\partial \gamma /\partial x = C
\prod_{j=0}^{\kappa-1} (x-\beta'_j)$ (with
sign$(C)=(-1)^{1+\kappa-\ell}$). Observe that
$\int_{\beta_\ell}^{\beta_{\ell+1}} \frac{\partial \gamma}
{\partial x} (x) \; dx = -1$.  So
$(-1)^{1+\kappa-\ell}/C=
\int_{\beta_\ell}^{\beta_{\ell+1}} (-1)^{\kappa-\ell}
\prod_{j=0}^{\kappa-1} (x-\beta'_j) \; dx $. Observe
that if for any $j<\ell$, $\beta'_j$ is increased, or
if for any $j \geq \ell$, $\beta'_j$ is decreased, then
the integral decreases. So $(-1)^{1+\kappa-\ell}/C
\geq \int_{\beta_\ell}^{\beta_{\ell+1}}
(-1)^{\kappa-\ell}
(x-\beta_\ell)^{\ell-1}(x-\beta_{\ell+1})^{\kappa-\ell}
\; dx $.  This is a definite integral that can be
evaluated in closed form:
$$
\int_{\beta_\ell}^{\beta_{\ell+1}} (-1)^{\kappa-\ell}
(x-\beta_\ell)^{\ell-1}(x-\beta_{\ell+1})^{\kappa-\ell}\; dx 
= (\beta_{\ell+1}-\beta_\ell)^{\kappa}
(\ell-1)!(\kappa-\ell)!/\kappa!\ .
$$
Hence, $(-1)^{1+\kappa-\ell}C \leq \frac{\kappa!}{s^\kappa(\ell-1)!(\kappa-\ell)!}$.
The sum of squares of coefficients of $\frac{\partial
\gamma}{\partial x}$ is $C^2
\sum_{i=0}^{\kappa-1} {\kappa-1 \choose i}^2
(e^{\kappa-1}_i(\beta'_1,\ldots,\beta'_{\kappa-1}))^2
\leq C^2 {2\kappa-2 \choose \kappa-1}$. Integration
only decreases the magnitude of the coefficients, so
the same bound applies to $\gamma$, with the
exception of the constant coefficient. The constant
coefficient can be bounded by the fact that $\gamma$
has a root in $(0,1)$, and that in that interval the
derivative is bounded in magnitude by $C
\sum_{i=0}^{\kappa-1} {\kappa-1 \choose i} =
C\cdot 2^\kappa$. So $|\gamma_0| \leq C\cdot 2^\kappa$.
Consequently, $\sum_{i=0}^\kappa \gamma_i^2$ is at most 
$$
C^2\left[{2\kappa-2 \choose \kappa-1}+2^{2\kappa}\right] 
\leq \frac{{2\kappa-2 \choose \kappa-1}+2^{2\kappa}}{s^{2\kappa}}
\cdot\left(\frac{\kappa!}{(\ell-1)!(\kappa-\ell)!}\right)^2 
\leq \frac{5 \kappa^2 2^{2\kappa-2}}{s^{2\kappa}}\cdot{\kappa-1 \choose \ell-1}^2 
\leq \frac{5 \kappa^2 2^{4\kappa-4}}{s^{2\kappa}},
$$
which completes the proof of the lemma. 
\end{proofof}

\begin{proofof}{Lemma~\ref{sens}} 
Recall that $G = G(g(\vth,\al))$ is the $k \times (k+1)$ matrix defined by
$G_{ij}=g(\vth,\al)_{i+j}$ for $i=0,\ldots,k-1$ and $j=0,\ldots,k$; 
$\ld$ is such that $G \lambda = 0$ and $\ld_k=1$; 
$\Ld=\Ld(\ld)$ is the $2k \times k$ matrix with $\Lambda_{ij}=\lambda_{i-j}$ 
(for $0 \leq i<2k,\ 0\leq j<k$  with the understanding $\lambda_\ell=0$ for $\ell \notin
\{0,\ldots,k\}$; and
$P_\ld(x)$ is the polynomial $\sum_{\ell=0}^k \lambda_\ell x^\ell$.
We use $V_k, V_{2k},$ to denote $V_k(\al), V_{2k}(\al)$ respectively, and $\tV_k,
\tV_{2k}, \tG, \tLd$ to denote $V_k(\tal), V_{2k}(\tal), G(\tg), \Ld(\tld)$ respectively. 
We abbreviate $g(\vth,\al)$ to $g$.

\begin{lemma} \label{cllambda}
If $\|\tg-g\|_2\leq\xi$, then $\|G\tld\|_1\leq 2^{k+1}k\xi$.
\end{lemma}

\begin{proof} 
First, observe that $\tG\ld=G\ld+(\tG-G)\ld=(\tG-G)\ld$.
Also $\|\ld\|_2\leq\|\ld\|_1=\prod_{i=1}^k(1+\al_i)\leq 2^k$.
The last two inequalities follows since $P_\ld(x)=\prod_{i=1}^k(x-\al_i)$, and
$P_{\ld}(-1)=(-1)^k\|\ld\|_1$. 
So for any $i=1,\ldots,k$,
$\bigl|(G-\tG)_i\cdot\ld\bigr|\leq\|\ld\|_2\|G_i-\tG_i\|_2\leq 2^k\xi$.
Thus, $\ld$ is a feasible solution to \eqref{ld-lp}, which implies that
$\|\tld\|_1\leq 2^k$.
We have $\|G\tld\|_1\leq\|\tG\tld\|_1+\|(G-\tG)\tld\|_1\leq
2^k k \xi+\|(G-\tG)\tld\|_1$.
For any $i=1,\ldots,k$,
$\bigl|(G-\tG)_i\cdot\tld\bigr|\leq\|G_i-\tG_i\|_2\|\tld\|_2\leq 2^k\xi$,
so $\|G\tld\|_1\leq 2^{k+1}k\xi$.
\end{proof}

\begin{lemma} \label{clroots}
For every $\al_i$, $i=1,\ldots,k$, there exists a
$\sg(i)\in\{1,\ldots,k\}$ such that
$\vth_i|\al_i-\tal_{\sg(i)}|\leq\frac{8}{\onedzeta}(2k\xi)^{1/k}$.
\end{lemma}

\begin{proof}
Since $\|G\tld\|_2\leq 2^{k+1}k\xi$ (by
Lemma~\ref{cllambda}),  we have equivalently that the
$\|.\|_2$ norm of $g\tLd=\vth V_{2k}\tLd$ is at most
$2^{k+1}k\xi$. We may write
$\vth V_{2k}\tLd$ as
$$
\vth V_{2k}\tLd=
\vth
\left(\begin{array}{cccc}
P_\tld(\al_1) & \al_1P_\tld(\al_1) & \cdots & \al_1^{k-1}P_\tld(\al_1) \\
P_\tld(\al_2) & \al_2P_\tld(\al_2) & \cdots & \al_2^{k-1}P_\tld(\al_2) \\
\vdots & \vdots & \ddots & \vdots \\
P_\tld(\al_k) & \al_kP_\tld(\al_k) & \cdots & \al_k^{k-1}P_\tld(\al_k) \\
\end{array}\right)
$$
which is equal to $\vth'V_k(\al)$ where
$\vth'=\bigl(\vth_1P_\tld(\al_1),\cdots,\vth_kP_\tld(\al_k)\bigr)$.
Thus, we are given that $\|\vth'V_k\|_2\leq 2^{k+1}k\xi$.

Let $(\gm^i)^\dagger=\bigl(\arg\min_{y\in\R^k:y_i=1}\|yV_k\|_2\bigr)V_k$. 
Then, we also have $\|\vth'V_k\|_2\geq\max_i|\vth'_i|\|\gm^i\|_2$.
Note that $\gm^i$ must be orthogonal to
$(V_k)_{j*}$ for all $j\neq i$, and $(V_k)_{i*}\gm^i=\|\gm^i\|_2^2$. 
(Recall that $Z_{i*}$ denotes row $i$ of a matrix $Z$.) 
Let $Q_i(x)=\sum_{\ell=0}^{k-1}\gm^i_\ell x^\ell$. 
Then, $Q_i(x)=\|\gm^i\|_2^2\prod_{j\neq i}\frac{x-\al_{j}}{\al_i-\al_{j}}$.
Also, since the coefficients of $Q_i(x)$ have alternating signs, we have 
\[
|Q_i(-1)|=\|\gm^i\|_1=\|\gm^i\|_2^2\prod_{j\neq i}\frac{1+\al_{j}}{|\al_i-\al_{j}|}.
\]
Hence, $\|\gm^i\|_2\geq\prod_{j\neq i}\frac{|\al_i-\al_j|}{1+\al_j}$.
So we obtain the lower bound
$$
\|\vth'V_k\|_2 \geq
\max_i\Bigl(|\vth'_i|\cdot\prod_{j\neq i}\frac{|\al_i-\al_j|}{1+\al_j}\Bigr) 
\geq
\max_i\biggl(\vth_i\Bigl(\frac{\onedzeta}{2}\Bigr)^{k-1}\prod_{j=1}^k|\al_i-\bal_{j}|\biggr) 
\geq
\max_i\biggl(\vth_i\Bigl(\frac{\onedzeta}{2}\Bigr)^{k-1}\prod_{j=1}^k|\al_i-\real(\bal_{j})|\biggr).
$$
The last inequality follows since complex roots occur in conjugate pairs, so if
$\bal_\ell=a+bi$ is complex, then there must be some $\ell'$ such that
$\bal_{\ell'}=a-bi$ and therefore,
$$
\prod_j |\al_i-\bal_{j}| =
\bigl((\al_i-a)^2+b^2\bigr)\cdot\prod_{j\neq\ell,\ell'}|\al_i-\bal_j| 
\geq (\al_i-a)^2\cdot\prod_{j\neq\ell,\ell'}|\al_i-\bal_j|.
$$
Now, we claim that $|\al_i-\real(\bal_j)|\geq\bigl||\al_i-\tal_j|-\e\bigr|$ for every
$j$. If both $\real(\bal_j)$ and $\real(\hal_j)$ lie in $[0,1]$, or both of them are
less than 0, or both are greater than 1, then this follows since $|\bal_j-\hal_j|\leq\e$
and $\al_i\in[0,1]$. If $\real(\bal_j)\notin[0,1]$ but $\real(\hal_j)\in[0,1]$, or if
$\real(\bal_j)\in[0,1]$ but $\real(\hal_j)\notin[0,1]$, then this again follows since
$|\bal_j-\hal_j|\leq\e$.
Combining everything, we get that
$$
2^k(2k\xi)\ \geq\ \|\vth'V_k\|_2\ \geq\
\max_i\biggl(\vth_i\Bigl(\frac{\onedzeta}{2}\Bigr)^{k-1}
\prod_{j=1}^k\bigl||\al_i-\tal_{j}|-\e\bigr|\biggr).
$$
This implies that for every $i=1,\ldots,k$, there exists
$\sg(i)\in\{1,\ldots,k\}$ such that
$\vth_i|\al_i-\tal_{\sg(i)}|\leq
\frac{4}{\onedzeta}\cdot\bigl(2k\xi\bigr)^{1/k}+\e$.
\end{proof}

We can now wrap up the proof of Lemma~\ref{sens}.
Let $\eta=\frac{8}{\onedzeta}\cdot\bigl(2k\xi\bigr)^{1/k}$.
We will bound $\|\tvth\tV_{2k}-\tg\|_2$ by exhibiting a
solution $y\in[0,1]^k,\ \|y\|_1=1$ such that
$\|y\tV_{2k}-\tg\|_2\leq\|g-\tg\|+k(8k)^{3/2}\eta$.
Let $\sg$ be the function whose existence is proved in Lemma~\ref{clroots}.
For $j=1,\ldots,k$, set $y_j=\sum_{i:\sg(i)=j}\vth_i$ (if
$\sg^{-1}(j)=\es$, then $y_j=0$).
We have $\|y\tV_{2k}-\tg\|_2\leq\|g-\tg\|_2+\|g-y\tV_{2k}\|_2$.
We expand $g-y\tV_{2k}=\vth V_{2k}-y\tV_{2k}
=\sum_{i=1}^k\vth_i\bigl((V_{2k})_{i*}-(\tV_{2k})_{\sg(i)*}\bigr)$ 
For every $i$,
$$
\vth_i^2\|(V_{2k})_{i*}-(\tV_{2k})_{\sg(i)*}\|_2^2=
\vth_i^2\sum_{\ell=0}^{2k-1}(\al_i^\ell-\tal_{\sg(i)}^\ell)^2
\leq \vth_i^2\cdot 8k^3\cdot \eta^2.
$$
Therefore, $\|g-y\tV_{2k}\|_2\leq k(8k)^{3/2}\eta$. 
\end{proofof}

\section{Lower bounds} \label{lbound}
In this section, we prove sample-size and aperture lower bounds that apply even to the
setting where we have $k$-mixture sources on $\{0,1\}$ (so $n=2$). 
Recall that a $k$-mixture source on $\{0,1\}$ may be equivalently viewed as a
$k$-spike distribution supported on $[0,1]$; in the sequel, we therefore focus on
$k$-spike distributions. The separation of a $k$-spike distribution (or the equivalent
$k$-mixture source) is the minimum separation between its spikes. 
Theorem~\ref{newmombnd} proves that $2k-1$ is the smallest aperture
at which it becomes possible to reconstruct a $k$-spike distribution. We emphasize that
this is an {\em information-theoretic} lower bound. We show (Theorem~\ref{newmombnd}) that
there are two $k$-spike distributions supported on $[0,1]$ having separation
$\Omega\bigl(\frac{1}{k}\bigr)$ and transportation distance
$\Omega\bigl(\frac{1}{k}\bigr)$ that yield exactly the same first $2k-2$ moments.  
Moreover, for any $b\geq 2k-1$, by adjusting the constant in the $\Omega(.)$s, one can 
ensure that the $(2k-1)$-th, \ldots, $b$-th moments of these two $k$-spike
distributions are exponentially close. 

It follows immediately that even with infinite sample size it is impossible to reconstruct
a $k$-mixture source (with arbitrarily small error) if we limit the aperture to $2k-2$.
Furthermore, we leverage the exponential closeness of the moments to show that for
any aperture $b\geq 2k-1$, there exists 
$\onedzeta=\Omega\bigl(\frac{1}{k}\bigr)$ such that reconstructing a $k$-mixture source on   
$\{0,1\}$ having separation $\onedzeta$ to within transportation distance
$\frac{\onedzeta}{4}$ requires {\em exponential in $k$} sample size 
(Theorem~\ref{newsbound}). 
In fact, since $n=2$, this means that with
arbitrary mixtures, the exponential dependence of the sample size on $k$ remains 
{\em even with aperture $O(k\log n)$}, and more generally, even with aperture
$O\bigl(k\cdot\kp(n)\bigr)$ for any function $\kp(.)$.  
(To place this in perspective, observe that with separation
$\onedzeta=\Omega\bigl(\frac{1}{k}\bigr)$, if we have
$\Omega(k^2\log k)$ aperture, then $O(k^3)$ samples suffice to reconstruct the given
mixture source to within transportation distance $\frac{\onedzeta}{4}$.
This is because with, with high probability, we will see every $\{0,1\}$ source or
``coin'' with weight $\vth_i\geq\frac{1}{\onedzeta^2}$, and we can estimate its bias
within additive error, say $\frac{\onedzeta}{8}$, with probability $1-\frac{1}{\poly(k)}$
since we have $\Omega\bigl(\frac{\log k}{\onedzeta^2}\bigr)$ coin tosses available. The
unseen coins contribute $O(\onedzeta)$ to the transportation 
distance, so we infer $k$-mixture source within transportation distance
$\frac{\onedzeta}{4}$.) 

\begin{theorem} \label{newsbound}
(i) With aperture $2k-2$, it is impossible to reconstruct a $k$-mixture source having
separation at least $\frac{1}{2k}$ to within transportation distance $\frac{1}{8k}$ even
with infinite sample size. 

\noindent
(ii) For any $\onederrp\in(0,1)$, and any constants $c_A\geq 1,\ c_E\geq 0$,
there exists $\onedzeta=\Omega\bigl(\frac{1}{k}\bigr)$ such that reconstructing a
$k$-mixture source having separation $\onedzeta$ to within transportation distance
$\frac{\onedzeta}{4}$ with probability at least $1-\onederrp$ using aperture $c_A(2k-1)$
requires $\Omega\bigl(3^{c_Ek}\ln(\frac{1}{\onederrp})\bigr)$ samples.
\end{theorem}

Our approach for proving Theorem~\ref{newsbound} is as follows. To prove the existence of 
two suitable $k$-spike distributions (Theorem~\ref{newmombnd}), we fix some spike
locations ensuring the required separation and transportation-distance bounds, and search
for suitable probability weights to place on these locations so as to obtain the desired
closeness of moments for the two $k$-spike distributions. Since moments are linear
functions of the weights (and the spike locations are fixed), this search problem can be
encoded as a minimization LP \eqref{newmomp}. 
To upper bound the optimum, we move to the dual LP \eqref{newmomd},
which can be interpreted as finding a polynomial satisfying certain conditions on its
coefficients and roots, to maximize the variation between its values at certain spike
locations. We upper bound the variation possible by such a polynomial using elementary
properties of polynomials.  
Finally, the closeness of moments of the two $k$-spike distributions obtained this way 
implies that the distributions of $b$-snapshots of these two distributions have
exponentially small variation distance (Lemma~\ref{newgtovar}), and this yields
the sample-size lower bound in Theorem~\ref{newsbound}. 

\begin{theorem} \label{newmombnd}
Let $b\geq 2k-1$, and $\rho\geq 2$. There are two $k$-spike distributions $(y,\al)$ and
$(z,\beta)$ on $[0,1]$ with separation $\frac{2}{(2k-1)\rho}$ and
$\Tran(y,\al;z,\beta)\geq\frac{1}{(2k-1)\rho}$ such that 
$g_\ell(y,\al)=g_\ell(z,\beta)$ for all $\ell=0,\ldots,2k-2$, and
$\sum_{\ell=2k-1}^{b}{b\choose \ell}2^\ell|g_{\ell}(y,\al)-g_{\ell}(z,\beta)|
\leq 4\cdot\frac{3^b}{\rho^{2k-1}}$.
\end{theorem}

\begin{proof}
Let $\e=\frac{1}{\rho}$.
We set $\al_i=\e\cdot\frac{2(i-1)}{2k-1}$, and
$\beta_i=\e\cdot\frac{2i-1}{2k-1}=\al_i+\frac{\e}{2k-1}$ for all $i=1,\ldots,k$. 
Note that for any mixture weights $y_1,\ldots,y_k$, and $z_1,\ldots,z_k$, the separation
of $(y,\al)$ and $(z,\beta)$ is $\frac{2}{(2k-1)\rho}$ and
$\Tran(y,\al;z,\beta)\geq\frac{1}{(2k-1)\rho}$.   
We obtain $y$ and $z$ by solving the following
linear program \eqref{newmomp}, whose optimal value
we show is at most $4\cdot\frac{3^b}{\rho^{2k-1}}$.

\vspace*{-3ex}
\hspace*{-0.55in}
\begin{minipage}[t]{.58\textwidth}
\begin{alignat}{2}
\min & \quad & \quad \sum_{\ell=2k-1}^{b}{b \choose \ell}2^{\ell}\ld_\ell 
\qquad \tag{P1} \label{newmomp} \\
\text{s.t.} & 
& \sum_{i=1}^k \bigl(z_i\beta_i^\ell - y_i\al_i^\ell) = 0 \qquad & \forall
\ell=0,\ldots,2k-2 \label{moments} \\[-1ex]
&& \sum_{i=1}^k \bigl(z_i\beta_i^{\ell} - y_i\al_i^{\ell}) \leq \ld_\ell 
\qquad & \forall \ell=2k-1,\ldots,b \label{mbnd1} \\[-1ex]
&& \sum_{i=1}^k \bigl(y_i\al_i^{\ell} - z_i\beta_i^{\ell}) \leq \ld_\ell 
\qquad & \forall \ell=2k-1,\ldots,b \label{mbnd2} \\[-1ex]
&& \sum_{i=1}^k y_i = 1, \quad 
y_i,z_i \geq 0 \qquad & \forall i=1,\ldots,k. \notag
\end{alignat}
\end{minipage}
\ \ \rule[-36ex]{1pt}{33ex} \ \
\begin{minipage}[t]{.48\textwidth}
\begin{alignat*}{2}
\max && \quad c \tag{D1} \label{newmomd} \\
\text{s.t.} & \ \ & c - 
\sum_{\ell=0}^{2k-2}\gm_\ell\al_i^\ell & -\sum_{\ell=2k-1}^{b}(\gm_{\ell}-\tht_{\ell})\al_i^{\ell} \leq 0 
\\[-1.5ex] &&& \qquad \qquad \qquad \forall i=1,\ldots,k \\
&& \sum_{\ell=0}^{2k-2}\gm_\ell\beta_i^\ell & +\sum_{\ell=2k-1}^{b}(\gm_{\ell}-\tht_{\ell})\beta_i^{\ell} \leq 0 
\\[-1.5ex] &&& \qquad \qquad \qquad \forall i=1,\ldots,k \\
&& \gm_{\ell}+\tht_{\ell} \leq {b \choose \ell} & 2^\ell 
\qquad \forall \ell=2k-1,\ldots,b \\ 
&& \gm_\ell, \tht_\ell & \geq 0 \qquad \forall \ell.
\end{alignat*}
\end{minipage}

\medskip
\eqref{newmomd} is the dual of \eqref{newmomp}. The dual variable $c$ corresponds to
$\sum_i y_i=1$, variables $\gm_\ell$ for $\ell=0,\ldots,b$ correspond to
\eqref{moments} and \eqref{mbnd2}, and variables $\tht_\ell$ for $\ell=2k-1,\ldots,b$
correspond to \eqref{mbnd1}. 
Given a feasible solution to \eqref{newmomd}, if we set 
$\gm'_{\ell}=\gm_\ell-\min\{\gm_{\ell},\tht_{\ell}\},\ \tht'_{\ell}=\tht_\ell-\min\{\gm_\ell,\tht_{\ell}\}$
for all $\ell=2k-1,\ldots,b$, then
we obtain another feasible solution to \eqref{newmomd}, where
$\gm'_{\ell}+\tht'_{\ell}=|\gm'_{\ell}-\tht'_{\ell}|=|\gm_\ell-\tht_\ell|$. 
Thus, an optimal solution to \eqref{newmomd} can be interpreted as a polynomial
$f(x)=\sum_{\ell=0}^{b}f_\ell x^\ell$ satisfying $|f_{\ell}|\leq{b \choose\ell}2^\ell$
for all $\ell=2k-1,\ldots,b$,
and $f(\al_i)\geq c>0\geq f(\beta_i)$ for all $i=1,\ldots,k$ (where $c>0$
follows from Lemma~\ref{infrecon}). 

Let $c'=3^{b}\cdot\frac{\rho/(\rho-1)}{\rho^{2k-1}}\leq 2\cdot\frac{3^b}{\rho^{2k-1}}$. 
Suppose that $c>2c'$.
Observe that for any $x\in[0,\e]$, by the Cauchy-Schwarz inequality and since the 
$\ell_2$-norm is at most the $\ell_1$ norm, we have
\begin{equation}
\Bigl|\sum_{\ell=2k-1}^{b}f_\ell x^\ell\Bigr|
\leq\Bigl(\sum_{\ell=2k-1}^{b}|f_\ell|\Bigr)\Bigl(\sum_{\ell=2k-1}^{b}x^\ell\Bigr)
\leq\sum_{\ell=2k-1}^{b}{b\choose \ell}2^{\ell}\cdot\frac{\rho/(\rho-1)}{\rho^{2k-1}}
\leq c'. \label{bnd1}
\end{equation}
Let $h(x)=\sum_{\ell=0}^{2k-2}f_\ell x^\ell-c/2$. Then, due to \eqref{bnd1}, we have
$f(x)-c/2-c'\leq h(x)\leq f(x)-c/2+c'$ for all $x\in[0,\e]$, 
so $h(\al_i)>0>h(\beta_i)$ for all $i=1,\ldots,k$. But then $h(x)$ has $2k-1$
roots---one in every $(\al_i,\beta_i)$ and $(\beta_i,\al_{i+1})$ interval---which is
impossible since $h(x)$ is a polynomial of degree $2k-2$. 
\end{proof}

Given a $k$-spike distribution $\bigl(\vth,\al=(\al_1,\ldots,\al_k)\bigr)$ on $[0,1]$, we 
abuse notation and denote the equivalent $k$-mixture source on $\{0,1\}$ also by
$(\vth,\al)$; that is, $\bdth=(\vth,\al)$ represents a mixture of
$k$ ``coins'', where coin $i$ has bias $\al_i$ and is chosen with weight $\vth_i$. 
Let $g(\vth,\al)=\bigl(g_i(\vth,\al)\bigr)_{i=0}^{2k-1}$.
We use $\dist{\bdth}$ (viewed as a vector in $\R_{\geq 0}^{\{0,1\}^{2k-1}}$) to denote the
distribution of $(2k-1)$-snapshots induced by $\bdth$ on $\{0,1\}^{2k-1}$.
The total variation distance $\dtv(\dist{\bdy},\dist{\bdz})$ between two such
distributions is defined to be 
$\frac{1}{2}\|\dist{\bdy}-\dist{\bdz}\|_1$.

\begin{lemma} \label{newgtovar}
Let $b\geq 2k-1$.
Given two $k$-mixture sources $\bdy=(y,\al)$ and $\bdz=(z,\beta)$ on $\{0,1\}$ with
identical first $2k-2$ moments, we have 
$\dtv(\dist{\bdy,b},\dist{\bdz,b})
=\frac{1}{2}\sum_{\ell=2k-1}^{b}{b\choose\ell}2^\ell|g_\ell(y,\al)-g_\ell(z,\beta)|$.
\end{lemma}

\begin{proof}
For any $s\in\{0,1\}^{b}$ with $i$ 1s, we have $\dist{\bdy,b}_s=\nu_i(y,\al)$ and
$\dist{\bdz,b}_s=\nu_i(z,\beta)$. 
Therefore, 
$\dtv(\dist{\bdy,b},\dist{\bdz,b})
=\frac{1}{2}\sum_{i=0}^{b}\binom{b}{i}|\nu_i(y,\al)-\nu_i(z,\beta)|$.
Let $B$ be the $(b+1)\times(b+1)$ diagonal matrix with $B_{ii}=\binom{b}{i}$ for 
$0\leq i\leq b$. 
Then,
$\binom{b}{i}\nu_i(y,\al)=\bigl(g^b(y,\al)\Pas_{b+1}^{-1}B\bigr)_i$. 
Let $\Dt g^b:=g^b(y,\al)-g^b(z,\beta)$.
So $\dtv(\dist{\bdy,b},\dist{\bdz,b})=\frac{1}{2}\|(\Dt g^b)(\Pas_{b+1}^{-1}B)\|_1$.
We prove below that $\Pas_{b+1}^{-1}$ is the lower triangular matrix with entries
$(\Pas_{b+1}^{-1})_{ij}=(-1)^{i-j}\binom{b-j}{i-j}$ for $0\leq j\leq i\leq b$ (and 0
otherwise). Let $Z_{i*}$ denote row $i$ of matrix $Z$.
Since $(\Dt g^b)_i=0$ for $i=0,\ldots,2k-2$, we have that
$\dtv(\dist{\bdy,b},\dist{\bdz,v})$ is at most
\begin{equation*}
\begin{split}
\frac{1}{2}\sum_{\ell=2k-1}^{b}|(\Dt g^b)_\ell|\cdot\|(\Pas^{-1}_{b+1} & B)_{\ell,*}\|_1
=\frac{1}{2}\sum_{\ell=2k-1}^{b}|(\Dt g^b)_\ell|\sum_{j=0}^\ell{b-j \choose\ell-j}{b\choose j} \\
& =\frac{1}{2}\sum_{\ell=2k-1}^{b}|(\Dt g^b)_\ell|{b\choose\ell}\sum_{j=0}^\ell{\ell\choose j}
=\frac{1}{2}\sum_{\ell=2k-1}^{b}|(\Dt g^b)_\ell|{b\choose\ell}2^\ell. \qedhere
\end{split} 
\end{equation*}
To see the claim about $\Pas_{b+1}^{-1}$, let $Q$ be the claimed inverse matrix; so 
$Q_{ij}=(-1)^{i-j}\binom{b-j}{i-j}$ for $0\leq j\leq i\leq b$. 
Then $(\Pas_{b+1}Q)_{ij}=0$ for $j>i$, and is equal to
$\sum_{\ell=j}^i{b-\ell \choose i-\ell}(-1)^{\ell-j}{b-j \choose \ell-j}$
otherwise. The latter expression evaluates to 
${b-j \choose i-j}\sum_{\ell=j}^i(-1)^{\ell-j}{i-j \choose \ell-j}
={b-j \choose i-j}(1-1)^{i-j}$, which is 0 if $i>j$, and 1 if $i=j$.
\end{proof}

\begin{proofof}{Theorem~\ref{newsbound}}
For part (i), take $\rho=2$ and $b=2k-1$. Consider the two $k$-mixture sources
$\bdy=(y,\al)$ and $\bdz=(z,\beta)$ given by Theorem~\ref{newmombnd}, which have
separation $\frac{1}{2k-1}$ and transportation distance $\frac{1}{2(2k-1)}$.
For any $b'\leq 2k-2$, the distributions $\dist{\bdy,b'}$ and $\dist{\bdz,b'}$ are 
{\em identical} and hence indistinguishable even with infinitely many samples, but the
stated reconstruction task would allow us to do precisely this.  

For part (ii), 
set $\rho=3^{c_E+c_A}$, $b=c_A(2k-1)$, and $\onedzeta=\frac{2}{(2k-1)\rho}$. 
Let $\bdy=(y,\al)$ and $\bdz=(z,\beta)$ be as given by Theorem~\ref{newmombnd} (for this
$b, \rho$), which satisfy the required separation property.  
Suppose that we can perform the stated reconstruction task using $N$
$b$-snapshots. Then, we can distinguish between $\bdy$ and $\bdz$ with probability at
least $1-\onederrp$. But this probability 
is also upper bounded by 
$\bigl[1+\dtv\bigl((\dist{\bdy,b})^N,(\dist{\bdz,b})^N\bigr)\bigr]/2$, where $(\dist{\bdy,b})^N$
and $(\dist{\bdz,b})^N$ are the $N$-fold products of $\dist{\bdy,b}$ and $\dist{\bdz,b}$
respectively. Thus, $\dtv\bigl((\dist{\bdy,b})^N,(\dist{\bdz,b})^N\bigr)\geq 1-2\onederrp$. 
By Proposition 11 and Lemma 12 in~\cite{BaryossefKS01}
$$
N\geq\frac{1}{4\dtv(\dist{\bdy,b},\dist{\bdz,b})}\ln\biggl(\frac{1}{1-(1-2\onederrp)^2}\biggr)
\geq\frac{\rho^{2k-1}}{8\cdot 3^b}\ln\Bigl(\frac{1}{4\onederrp}\Bigr)
=\Omega\biggl(3^{c_E(2k-1)}\ln\Bigl(\frac{1}{\onederrp}\Bigr)\biggr)
$$
where the second inequality follows from Theorem~\ref{newmombnd} and Lemma~\ref{newgtovar}.
\end{proofof}

\bibliographystyle{plain}
\bibliography{ref}

\appendix

\section{Probability background} \label{probinfo}

We use the following large-deviation bounds in our analysis.

\begin{lemma}[\textnormal{Chernoff bound; see Theorem 1.1 in~\cite{DubhashiP09}}]
\label{chernoff}
Let $X_1,\ldots,X_N$ be independent random variables with $X_i\in[0,1]$ for all $i$, and 
$\mu=\bigl(\sum_i\E{X_i}\bigr)/N$. Then,
$\Pr\bigl[\bigl|\frac{1}{N}\sum_i X_i-\mu\bigr|>\e\bigr]\leq 2e^{-2\e^2N}$. 
\end{lemma}

\begin{lemma}[\textnormal{Bernstein's inequality; see Theorem 1.2
      in~\cite{DubhashiP09}}] \label{bernstein}
Let $X_1,\ldots,X_N$ be independent random variables with $|X_i|\leq b,\ \E[X_i]=0$ for
all $i$, and let $\sg^2=\sum_i\Var[X_i]$. Then,
$\Pr\bigl[|\sum_i X_i|>t\bigr]\leq 2\exp\bigl(-\frac{t^2}{2(\sg^2+bt/3)}\bigr)$.
\end{lemma}

\section{\boldmath Sample-size dependence of~\cite{AFHKL12,AHK12,AGM12} on $n$ for
  $\ell_1$-reconstruction} \label{comparison} 
We view $P=(p^1,\ldots,p^k)$ as an $n\times k$ matrix. Recall that 
$r=\sum_{t=1}^k w_tp_tp_t^\dagger$, $A=\sum_{t=1}^kw_t(p^t-r)(p^t-r)^\dagger$, and
$M=rr^\dagger+A$. Let $w_{\max}:=\max_t w_t$. 
We consider isotropic $k$-mixture sources, which is justified by 
Lemma~\ref{lm: isotropy}. So $\frac{1}{2n}\leq r_i\leq\frac{2}{n}$ for all $i\in[n]$.     
Note that $\|r\|_1$, $\|r\|_2^2$, and $\|r\|_\infty$ are all
$\Theta\bigl(\frac{1}{n}\bigr)$. 
It will be convenient to split the width parameter $\zeta$ into two parameters.
Let (i) $\frac{\zeta_1}{\sqrt{n}}=\min_{p,q\in P,p\neq q}\|p-q\|_2$; and
(ii) $\zeta_2^2\|r\|_\infty$ be the smallest non-zero eigenvalue of $A$. 
Then, the width of $(w,P)$ is $\zeta=\min\{\zeta_1,\zeta_2\}$. 
We use $\sg_i(Z)$ to denote the $i$-th largest singular value of a matrix $Z$. 
If $Z$ has rank $\ell$, its condition number is given by
$\kp(Z):=\sg_1(Z)/\sg_{\ell}(Z)$. 
For a square matrix $Z$ with real eigenvalues, we use $\ld_i(Z)$ to denote the $i$-th 
largest eigenvalue of $Z$. Note that if $Z$ is an $n\times k$ matrix, then
$\sg_i(Z)^2=\ld_i(ZZ^\dagger)=\ld_i(Z^\dagger Z)$ for all $i=1,\ldots,k$.
Also the singular values of $ZZ^\dagger$ coincide with its eigenvalues, and the same holds 
for $Z^\dagger Z$.

We now proceed to evaluate the sample-size dependence of~\cite{AFHKL12,AHK12,AGM12} on $n$
for reconstructing the mixture constituents within $\ell_1$-distance $\e$. 
Since these papers use different parameters than we do, in order to obtain a meaningful
comparison, we relate their bounds to our parameters $\zeta_1, \zeta_2$; we keep track of
the resulting dependence on $n$ but ignore the (polynomial) dependence on other
quantities. We show that the sample size needed is at least
$\Omega\bigl(\frac{n^4}{\e^2}\bigr)$, with the exception of Algorithm B in~\cite{AHK12}, 
which needs $\Omega\bigl(\frac{n^3}{\e^2}\bigr)$ samples. 
As required by~\cite{AFHKL12,AHK12,AGM12}, we assume that $P$ has full column rank. 
It follows that $M$ has rank $k$ and $A$ has rank $k-1$. 
The following inequality will be useful.

\begin{proposition} \label{evineq}
Let $D=\diag(d_1,\ldots,d_k)$ where $d_1\geq d_2\geq\ldots\geq d_k>0$. 
Then $\ld_k(PDP^{\dagger})\geq d_k\ld_k(PP^{\dagger})=d_k\sg_k(P)^2$.
\end{proposition}

{\bf Comparison with~\cite{AGM12}.\ } 
The algorithm in~\cite{AGM12} requires also that $P$ be $\rho$-separable. This means that
for every $t\in[k]$, there is some $i\in[n]$ such that $p^t_i\geq\rho$ and $p^{t'}_i=0$
for all $t'\neq t$. This has the following implications. 
For any $t,t'\in[k],\ t\neq t'$, we have $\|p^t-p^{t'}\|_2\geq\sqrt{2}\rho$, so 
$\frac{\zeta_1}{\sqrt{n}}\geq\sqrt{2}\rho$. 
We can write $P^\dagger P=Y+Z$, where $Y$ is a PSD matrix, and $Z$ is a diagonal
matrix whose diagonal entries are at least $\rho^2$. 
So $\ld_k(P^\dagger P)=\ld_k(PP^\dagger)\geq\rho^2$. 
Therefore, 
$$
\zeta_2^2\|r\|_\infty+\|r\|_2^2=\ld_k(A)+\|r\|_2^2\geq\ld_k(M)\geq w_{\min}\cdot\rho^2
$$ 
where the first inequality follows from Lemma~\ref{evper}, and the second from
Proposition~\ref{evineq}.  
It follows that $\rho=O\bigl(\frac{1}{\sqrt{n}}\bigr)$. 
The bound in~\cite{AGM12} to obtain $\ell_\infty$ error $\veps$ is (ignoring dependence on
other quantities) 
$\Omega\bigl(\frac{1}{\veps^2\rho^6}\bigr)$. So setting $\veps=\frac{\e}{n}$ to guarantee
$\ell_1$-error at most $\e$ and plugging in the above upper bounds on $\rho$, we obtain
that the sample size is $\Omega\bigl(\frac{n^5}{\e^2}\bigr)$. 

\medskip
{\bf Comparison with~\cite{AFHKL12}.\ }
The sample size required by~\cite{AFHKL12} for the latent Dirichlet model for obtaining
$\ell_2$ error $\veps$ is $\Omega\bigl(\frac{1}{\veps^2\sg_k(P)^6}\bigr)$. 
Proposition~\ref{evineq} yields $\ld_k(M)\geq w_{\min}\cdot\sg_k(P)^2$
and as argued above, $\ld_k(M)\leq\ld_k(A)+\|r\|_2^2=O\bigl(\frac{1}{n}\bigr)$.
So $\sg_k(P)^6=O\bigl(\frac{1}{n^3}\bigr)$.
Setting $\veps=\frac{\e}{\sqrt{n}}$ for $\ell_1$ error $\e$, this yields a bound of
$\Omega\bigl(\frac{n^4}{\e^2}\bigr)$.

\medskip
{\bf Comparison with~\cite{AHK12}.\ }
Algorithm A in~\cite{AHK12} 
requires sample size $\Omega\bigl(\frac{1}{\sg_k(P)^8\sg_k(M)^4\veps^2}\bigr)$ to recover
each $p^t$ to within $\ell_2$-distance $\veps\max_{p\in P}\|p\|_2$.
Since $\max_{p\in P}\|p\|_2\leq\frac{2}{w_{\min}\sqrt{n}}$ due to isotropy, we can 
set $\veps=\frac{\e w_{\min}}{2}$ to obtain $\ell_1$-error $\e$. 
Since $\sg_k(P)^2$ and $\sg_k(M)=\ld_k(M)$ are both $O\bigl(\frac{1}{n}\bigr)$, we obtain
a bound of $\Omega\bigl(\frac{n^8}{\e^2}\bigr)$. 

\noindent
Algorithm B in~\cite{AHK12} uses sample size 
$\Omega\Bigl(\kp(P)^8/\bigl(\frac{\zeta_1^2}{n}\cdot\sg_k(M)^2\veps^2\bigr)\Bigr)$ 
to recover each $p^t$ to within $\ell_2$-distance $\veps\max_{p\in P}\|p\|_2$. 
Clearly $\kp(P)\geq 1$. 
Again, setting $\veps=\frac{\e w_{\min}}{2}$, 
this yields a sample size of $\Omega\bigl(\frac{n^3}{\e^2}\bigr)$ for $\ell_1$ error $\e$.   
\end{document}